\documentclass{article}

\usepackage[nonatbib, final]{neurips_2021}
\usepackage{natbib}
\usepackage{multirow}

\usepackage[utf8]{inputenc} 
\usepackage[T1]{fontenc}    
\usepackage{url}            
\usepackage{booktabs}       
\usepackage{amsfonts}       
\usepackage{nicefrac}       
\usepackage{microtype}      
\usepackage{xfrac}

\usepackage{soul}
\usepackage{xfrac}
\usepackage[usenames, dvipsnames]{xcolor}
\usepackage{enumitem}
\definecolor{darkblue}{rgb}{0.0, 0.0, 0.55}
\definecolor{maroon}{rgb}{0.5, 0.0, 0.0}
\usepackage{url}

\usepackage{booktabs}
\usepackage{graphicx}
\usepackage{natbib}
\usepackage{amsmath}
\usepackage{booktabs}
\usepackage{algorithmic}
\usepackage{breakcites}
\usepackage{circledsteps}
\usepackage{bbm}
\usepackage{natbib}
\usepackage{bm}

\usepackage{wrapfig}

\usepackage{amsmath, amsfonts}  
\usepackage{bm}
\usepackage[colorlinks, citecolor=darkblue,linkcolor=maroon]{hyperref}
\newcommand{\bx}{{\bm x}}

\newcommand{\bu}{{\bm u}}

\newcommand{\br}{{\bm r}}

\newcommand{\bLambda}{{\bm \Lambda}}
\newcommand{\blambda}{{\bm \lambda}}

\newcommand{\bU}{{\bf U}}

\newcommand{\bW}{{\bf W}}

\newcommand{\bX}{{\bf X}}

\newcommand{\bTheta}{{\bf \Theta}}

\newcommand{\btheta}{{\bm \theta}}

\newcommand{\bbR}{\mathbb R}

\newcommand{\bbE}{\mathbb E}
\newcommand{\bbV}{\mathbb V}
\newcommand{\bbZ}{\mathbb Z}
\newcommand{\cX}{\mathcal X}
\newcommand{\cH}{\mathcal H}

\newcommand{\cF}{\mathcal F}

\newcommand{\cR}{\mathcal R}
\newcommand{\cB}{\mathcal B}

\newcommand{\cO}{\mathcal O}

\newcommand{\cE}{\mathcal E}

\newcommand{\cU}{\mathcal U}
\newcommand{\cY}{\mathcal Y}
\newcommand{\cS}{\mathcal S}
\newcommand{\fP}{\mathfrak P}

\newcommand{\fR}{\mathfrak R}

\newcommand{\cP}{\mathcal P}

\usepackage{mathtools}

\newcommand\xleftarrowbig[1][2ex]{%
   \mathrel{\rotatebox{0}{$\xleftarrow{\rule{#1}{0pt}}$}}
}\newcommand\xrightarrowbig[1][2ex]{%
   \mathrel{\rotatebox{180}{$\xleftarrow{\rule{#1}{0pt}}$}}
}
\usepackage[compact]{titlesec}         
\titlespacing{\section}{0pt}{0pt}{0pt} 
\AtBeginDocument{
  \setlength\abovedisplayskip{0pt}
  \setlength\belowdisplayskip{0pt}}

\DeclareMathOperator{\sign}{sign}
\usepackage{amsmath}
\usepackage{amssymb}
\usepackage{epsfig, psfrag}
\usepackage{amsthm}
\usepackage{latexsym}
\usepackage{bbm}
\usepackage{soul}

\usepackage[ruled,lined]{algorithm2e}
\SetKw{Init}{Initialize:}
\SetKw{Def}{Definition:}
\SetKw{KwInit}{Initialize:}

\SetCommentSty{mycommfont}

\usepackage{amsfonts} 
\usepackage{dsfont} 
\usepackage{euscript}

\allowdisplaybreaks
\usepackage{mathtools}

\usepackage{epstopdf}
\usepackage{subfig}
\usepackage{graphics,color}
\usepackage{graphicx}
 \usepackage{xcolor}
 
\usepackage{lipsum}

\newtheorem{theorem}{Theorem}
\newtheorem{lemma}{Lemma}
\newtheorem{proposition}{Proposition}

\newtheorem{assumption}{Assumption}

\theoremstyle{definition}

\def\argmin{\mathop{\rm arg\,min}}

\usepackage{titlesec}
\titlespacing\section{0pt}{2pt plus 1pt minus 1pt}{1pt plus 1pt minus 1pt}
\titlespacing\subsection{0pt}{2pt plus 1pt minus 1pt}{1pt plus 1pt minus 1pt}

\title{Scalable Interpretability via Polynomials}

\author{Abhimanyu Dubey \\ Meta AI \\ {\tt dubeya@fb.com} \And Filip Radenovic \\ Meta AI \\ {\tt filipradenovic@fb.com} \And Dhruv Mahajan \\ Meta AI\\ {\tt dhruvm@fb.com} }

\begin{document}

\maketitle

\begin{abstract}
Generalized Additive Models (GAMs) have quickly become the leading choice for interpretable machine learning. However, unlike uninterpretable methods such as DNNs, they lack expressive power and easy scalability, and are hence not a feasible alternative for real-world tasks. We present a new class of GAMs that use tensor rank decompositions of polynomials to learn powerful, {\em inherently-interpretable} models. Our approach, titled Scalable Polynomial Additive Models (SPAM) is effortlessly scalable and models {\em all} higher-order feature interactions without a combinatorial parameter explosion. SPAM outperforms all current interpretable approaches, and matches DNN/XGBoost performance on a series of real-world benchmarks with up to hundreds of thousands of features. We demonstrate by human subject evaluations that SPAMs are demonstrably more interpretable in practice, and are hence an effortless replacement for DNNs for creating interpretable and high-performance systems suitable for large-scale machine learning.
Source code is available at \href{https://github.com/facebookresearch/nbm-spam}{\ttfamily github.com/facebookresearch/nbm-spam}. 
\end{abstract}
\section{Introduction}
Interpretable machine learning systems suffer from a tradeoff between approximation and interpretability: high-performing models used in practice have millions of parameters and are highly non-linear, making them difficult to interpret. {\em Post-hoc} explainability attempts to remove this tradeoff by explaining black-box predictions with interpretable instance-specific approximations (e.g., LIME~\citep{ribeiro2016should}, SHAP~\citep{lundberg2017unified}), however, they are notoriously unstable~\citep{ghorbani2019interpretation, lakkaraju2020fool}, expensive to compute~\citep{slack2021reliable}, and in many cases, inaccurate~\citep{lipton2018mythos}.  It is therefore desirable to learn models that are instead {\em inherently-interpretable} (glass-box), i.e., do not require {\em post-hoc} interpretability, but provide better performance than {\em inherently-interpretable} classical approaches such as linear or logistic regression.

This has led to a flurry of research interest in Generalized Additive Models~\citep{hastie2017generalized}, i.e., methods that non-linearly transform each input feature separately, e.g., tree-based GAM and GA$^2$Ms~\citep{lou2013accurate}, or NAMs~\citep{agarwal2021neural}. While these are an improvement over linear methods, due to their highly complex training procedures, they are not a simple drop-in replacement for black-box methods.
Furthermore, it is not straightforward to model feature interactions in these models: the combinatorial explosion in higher-order interactions makes this infeasible without a computationally expensive feature selection~\citep{lou2013accurate, chang2021node}. 

Contrary to these approaches, we revisit an entirely different solution: {\em polynomials}. It is folklore in machine learning that since polynomials are universal approximators~\citep{stone1948generalized, weierstrass1885analytic}, hypothetically, modeling all possible feature interactions will suffice in creating powerful learners and eliminate the need for intricate non-linear transformations of data. Furthermore, polynomials can also provide a rubric for interpretability: lower-degree polynomials, that have interactions between a few features are interpretable but likely less accurate, whereas polynomials of higher degree have larger predictive power at the cost of interpretability. Thus, not only are polynomial models capable of capturing all possible interactions between features, they also give practitioners the ability to select a suitable model (order of interactions, i.e., degree) based precisely on their requirement in the interpretability-performance tradeoff. However, such an approach has till date remained elusive, as learning interpretable feature interactions {\em efficiently and at scale} remains an open problem.



In this paper, we introduce a highly scalable approach to modeling feature interactions for inherently interpretable classifiers based on {\em rank decomposed polynomials}. Our contributions are: 
\begin{enumerate}[leftmargin=*,itemsep=1pt,topsep=0pt,parsep=0pt,partopsep=0pt]
    \item First, we present an algorithm titled Scalable Polynomial Additive Models ({\bf SPAM}) to learn \underline{\em inherently-interpretable} classifiers that can learn all possible feature interactions by leveraging {\em low-rank tensor decompositions} of polynomials~\citep{nie2017low}. To the best of our knowledge, {\bf SPAM} is the among the first Generalized Additive Models with feature interactions scalable to 100K features. SPAM can be trained end-to-end with SGD (backpropagation) and GPU acceleration, but has orders of magnitude fewer parameters than comparable interpretable models.
    \item 
    We demonstrate that under a coarse regularity assumption (Assumption~\ref{assumption:singular_value_decay}), SPAM converges to the optimal polynomial as the number of samples $n \rightarrow \infty$. Furthermore, we establish novel non-asymptotic excess risk bounds that match classic bounds for linear or full-rank polynomials.
    \item SPAM outperforms all current interpretable baselines on several machine learning tasks. 
    To the best of our knowledge, our experimental benchmark considers problems of size (number of samples, and dimensionality) orders of magnitude larger than prior work, and we show that our simple approach scales easily to all problems, whereas prior approaches fail. Moreover, 
    we demonstrate that the need for modeling feature interactions is as important as non-linear feature transformations: \underline{on most real-world tabular datasets, {\em pairwise} interactions suffice to match DNN performance}.
    \item 
    We conduct a detailed human subject evaluation to corroborate our performance with practical interpretability. We demonstrate that SPAM explanations are  more faithful compared to {\em post-hoc} methods, making them a {\em drop-in replacement for black-box methods} such as DNNs for interpretable machine learning on tabular and concept-bottleneck~\citep{koh2020concept} problems.
\end{enumerate}

\section{Related Work}
Here we survey the literature most relevant to our work, i.e., interpretable ML with feature interactions, and polynomials for machine learning. Please see Appendix Section~\ref{sec:appendix_related_work} for a detailed survey.

{\bf Transparent and Interpretable Machine Learning}. Early work has focused on greedy or ensemble approaches to modeling interactions~\citep{friedman2001greedy, friedman2008predictive} that enumerate pairwise interactions and learn additive interaction effects. Such approaches often pick spurious interactions when data is sparse~\citep{lou2013accurate} and are impossible to scale to modern-sized datasets due to enumeration of individual combinations. 
As an improvement,~\citet{lou2013accurate} proposed 
{\sc GA$^2$M} that uses a statistical test to select only ``true'' interactions. 
However {\sc GA$^2$M} fails to scale to large datasets as it requires constant re-training of the model and ad-hoc operations such as discretization of features which may require tuning for datasets with a large dimensionality. Other generalized additive models require expensive training of millions of decision trees, kernel machines or splines~\citep{hastie2017generalized}, making them unattractive compared to black-box models. 

An alternate approach is 
is to learn interpretable neural network transformations. Neural Additive Models (NAMs,~\citet{agarwal2021neural}) learn a DNN per feature.
TabNet~\citep{arik2021tabnet} and NIT~\citep{tsang2018neural} alternatively modify NN architectures to increase their interpretability. NODE-GAM~\citep{chang2021node} improves NAMs with oblivious decision trees for better performance while maintaining interpretability. 
Our approach is notably distinct from these prior works: we do not require iterative re-training; we can learn {\em all} pairwise interactions regardless of dimensionality; we can train SPAM via backpropagation; and we scale effortlessly to very large-scale datasets. 


{\bf Learning Polynomials}. The idea of decomposing polynomials was of interest prior to the deep learning era. Specifically, the work of~\citet{ivakhnenko1971polynomial, oh2003polynomial, 155142} study learning neural networks with polynomial interactions, also known as {\em ridge polynomial networks} (RPNs). However, RPNs are typically not interpretable: they learn interactions of a very high order, and include non-linear transformation. 
Similar rank decompositions have been studied in the context of matrix completion~\citep{recht2011simpler}, and are also a subject of interest in tensor decompositions~\citep{nie2017generating, brachat2010symmetric}, where, contrary to our work, the objective is to decompose existing tensors rather than directly learn decompositions from gradient descent. Recently, ~\citet{chrysos2019polygan, chrysos2020p} use tensor decompositions to learn higher-order polynomial relationships in intermediate layers of generative models.
However, their work uses a recursive formulation and learns high-degree polynomials directly from uninterpretable input data (e.g., images), and hence is non-interpretable.
\section{Scalable Polynomial Additive Models}
\label{sec:method}
{\bf Setup and Notation}. Matrices are represented by uppercase boldface letters, e.g., $\bX$ and vectors by boldface lowercase, i.e., $\bx$. We assume that the covariates lie within the set $\cX \subseteq \bbR^d$, and the labels lie within the finite set $\cY$. Data $(\bx, y) \in \cX\times\cY$ are drawn following some unknown (but fixed) distribution $\fP$. We assume we are provided with $n$ i.i.d. samples $\{(\bx_i, y_i)\}_{i=1}^n$ as the training set.

{\bf Motivation}. Generalized Additive Models~\citep{hastie2017generalized} are an excellent design choice for interpretable models, as they learn transformations of individual features, allowing us to model exactly the contribution of any feature. A typical GAM is as follows.
    \parbox{0.95\textwidth}{%
    \small
        {\begin{gather*}
\underset{\text{Interpretability}}{\xleftarrowbig[6cm]}\\
    f(\bx) = \sum_{i=1}^d \underbrace{f_i(x_i)}_{\text{order 1 (unary)}} + \sum_{j>i}^{d} \underbrace{f_{ij}(x_i, x_j)}_{\text{order 2 (pairwise)}} +\ \dots +\ \underbrace{f_{1...d}(\bx)}_{\text{order $d$}}.\\
\underset{\text{Performance}}{\xrightarrowbig[6cm]}
\end{gather*}
    }%
}
Where $f_i,$ etc., are possibly non-linear transformations. It is evident~\citep{lou2013accurate} that as the order of interaction increases, e.g., beyond pairwise interactions, these models are no longer interpretable, albeit at some benefit to performance. While some approaches (e.g.,~\citet{agarwal2021neural, chang2021node}) learn $f_i$ via neural networks, we want to learn the simplest GAMs, i.e., polynomials. Specifically, we want to learn a polynomial $P(\bx)$, of order (degree) $k \leq d$ of the form:
\begin{align*}
    P(\bx) = b + \sum_{i=1}^d w^{(1)}_i\cdot x_i + \sum_{i,j}^d w^{(2)}_{ij} \cdot x_ix_j + ... \sum_{i_1 ...,i_k}^d \left(w^{(k)}_{i_1...i_k} \cdot\prod_{j=1}^k x_{i_j}\right).
\end{align*}
Here, the weights $\bW^{(l)} = \{w^{(l)}_{i_1...i_l}\}, 1 \leq l \leq k$ capture the $l-$ary interactions between subsets of $l$ features. For small values of $d$ and $k$, one can potentially enumerate all possible interactions and learn a linear model of $\cO(d^k)$ dimensionality, however this approach quickly becomes infeasible for large $d$ and $k$. Furthermore, the large number of parameters in $P(\bx)$ make regularization essential for performance, and the computation of each interaction can be expensive. Alternatively, observe that any polynomial that models $k-$ary interactions can be written as follows, for weights $\{\bW^{(l)}\}_{l=1}^k$,
\begin{align*}
    P(\bx) = \bW^{(1)} \odot_1 \bx + \bW^{(2)} \odot_2 \bx + ... \bW^{(k)} \odot_k \bx + b. 
\end{align*}
Here, the weights $\bW^{(l)}\in \bbR^{d^l}$ are written as a tensor of order $l$, and the operation $\odot_l$ refers to a repeated inner product along $l$ dimensions, e.g., $\bW \in \bbR^{d^2}$, $\bx^\top\bW\bx = \bW \odot_2 \bx = \left(\bW \odot \bx\right) \odot \bx$ where $\odot$ denotes the inner product. More generally, any {\em symmetric} order $l$ tensor admits an equivalent polynomial representation having {\em only} degree$-l$ terms. Now, this representation is still plagued with the earlier curse of dimensionality for arbitrary weight tensors $\{\bW^{(l)}\}_{l=1}^k$, and we will now demonstrate how to circumvent this issue by exploiting rank decompositions.

\subsection{Learning Low-Rank Decompositions of Polynomials}
The primary insight of our approach is to observe that any {\em symmetric} tensor $\bW^{(l)}$ of order $l$ also admits an additional {\em rank decomposition}~\citep{nie2017generating, brachat2010symmetric} ($\otimes$ being the outer product):
\begin{align*}
    \bW^{(l)} = \sum_{i=1}^{r}\lambda_i \cdot\underbrace{\bu_i \otimes \bu_i  \dots \otimes \bu_i}_{l \text{ times}}.
\end{align*}
Where $\{\bu_i\}_{i=1}^r$ are the (possibly orthonormal) $d$ dimensional basis vectors, $r \in \bbZ_+$ denotes the {\em rank} of the tensor, and the scalars $\{\lambda_i\}_{i=1}^r, \lambda_i \in \bbR$ denote the singular values. Our objective is to \underline{directly learn the rank decomposition of an order $l$ tensor}, and therefore learn the required polynomial function. This gives us a total of $\cO(rd)$ learnable parameters per tensor, down from a previous dependence of $\cO(d^l)$. This {\em low-rank} formulation additionally enables us to compute the polynomial more efficiently. Specifically, for a degree $k$ polynomial with ranks $\br = \{1, r_2, ..., r_k\}$, we have:
\begin{align}
\label{eq:low_rank_rep}
    P(\bx) = b + \langle \bu_1, \bx\rangle + \sum_{i=1}^{r_2} \lambda_{2i}\cdot \langle \bu_{2i}, \bx\rangle^2 + \sum_{i=1}^{r_3} \lambda_{3i}\cdot \langle \bu_{3i}, \bx\rangle^3 ... + \sum_{i=1}^{r_k} \lambda_{ki}\cdot \langle \bu_{ki}, \bx\rangle^k.
\end{align}
Where $\{\bu_{li}\}_{i=1}^{r_l}, \{\lambda_{li}\}_{i=1}^{r_l}$ denote the corresponding bases and eigenvalues for $\bW^{(l)}$. The above term can now be easily computed via a simple inner product, and has a time as well as space complexity of $\cO(d\lVert \br \rVert_1)$, instead of the earlier $\cO(d^k)$. Here, $\lVert \br \rVert_1$ simply represents the sum of ranks of all the $k$ tensors. Note that {\em low-rank} decompositions do not enjoy a free lunch: the rank $r$ of each weight tensor can, in the worst case, be polynomial in $d$. However, it is reasonable to assume that correlations in the data $\cX$ will ensure that a small $r$ will suffice for real-world problems (see, e.g., Assumption~\ref{assumption:singular_value_decay}). 

{\bf Optimization}. For notational simplicity, let us parameterize any order $k$ polynomial by the weight set $\btheta = \{b, \bu_1, \{\lambda_{2i}, \bu_{2i}\}_{i=1}^{r_2}, ..., \{\lambda_{ki}, \bu_{ki}\}_{i=1}^{r_k}\} \in \bbR^{r(d+1)}$ where from now onwards $r = \lVert \br \rVert_1 = 1 + r_2 + r_3 + ... + r_k$ denotes the cumulative rank across all the tensors. We write the polynomial parametrized by $\btheta$ as $P(\cdot; \btheta)$, and the learning problem for any loss $\ell$ is:
\begin{align}
    \label{eqn:optimization_fn}
    \text{select }\btheta_\star = \argmin_{\btheta \in \bTheta} \sum_{i=1}^n \ell(P(\bx_i; \btheta), y_i) + \beta\cdot\cR(\btheta).
\end{align}
Here, $\bTheta$ denotes the feasible set for $\btheta$, and $\cR(\cdot)$ denotes an appropriate regularization term scaled by $\beta > 0$. We can show that the above problem is well-behaved for certain data distributions $\cX$. 

\begin{proposition}
If the regularization $\cR$ and loss function $\ell: \cY\times\cY\rightarrow [0,1]$ are convex, $\lambda > 0$ and $\cX \subset \bbR_{+}^d$ then the optimization problem in Equation~\ref{eqn:optimization_fn} is convex in $\btheta$ for $\bTheta \subset \bbR_{+}^{r(d+1)}$.
\label{prop:convexity}
\end{proposition}
Proposition~\ref{prop:convexity} suggests that if the training data is positive (achieved by renormalization), typical loss functions (e.g., cross-entropy or mean squared error) and $\cR$ such as $L_p$ norms are well-suited for optimization. To ensure convergence, one can use proximal SGD with a small learning rate~\citep{nitanda2014stochastic}. In practice, we find that unconstrained SGD also provides comparable performance.

\subsection{Improving Polynomials for Learning}
{\bf Geometric rescaling}. Learning higher-order interactions is tricky as the order of interactions increases: the product of two features is an order of magnitude different from the original, and consequently, higher-order products are progressively disproportionate. To mitigate this, we rescale features such that (a) the scale is preserved across terms, and (b) the variance in interactions is captured. We replace the input $\bx$ with $\tilde\bx_l = \sign(\bx)\cdot \left|\bx\right|^{1/l}$ for an interaction of order $l$. Specifically, 
\begin{align}
\label{eqn:rescaling}
  P(\bx) = \langle \bu_1, \bx\rangle + \sum_{i=1}^{r_2} \lambda_{2i}\cdot \langle \bu_{2i}, \tilde\bx_2\rangle^2 + \sum_{i=1}^{r_3} \lambda_{3i}\cdot \langle \bu_{3i}, \tilde\bx_3\rangle^3 ... + \sum_{i=1}^{r_k} \lambda_{ki}\cdot \langle \bu_{ki}, \tilde\bx_k\rangle^k + b.
\end{align}
We denote this model as {\sc SPAM-Linear}. We argue that this rescaling for unit-bounded features ensures higher interpretability as well as better learning (since all features are of similar variance). Note that for order 1, $\tilde\bx_1 = \bx$. For $k\geq 2$, consider a pairwise interaction between two features $x_i = 0.5$ and $x_j = 0.6$. Naively multiplying the two will give a feature value of $x_ix_j = 0.3$ (smaller than both $x_i$ and $x_j$), whereas an intuitive value of the feature would be $\sqrt{x_ix_j} = 0.54$, i.e., the geometric mean. Regarding the reduced variance, observe that unconstrained, the variance $\bbV(x_ix_j) \leq \bbV(x_i)\cdot\bbV(x_j)$ and $\bbV(\sqrt{x_ix_j}) \leq \bbV(\sqrt{x_i})\bbV(\sqrt{x_j}) \leq \sqrt{\bbV(x_i)\cdot\bbV(x_j)}$. If the features have small variance, e.g., $\bbV(x_i) = \bbV(x_j) = 10^{-2}$ and are uncorrelated, then the first case would provide a much smaller variance for the interaction, whereas the rescaling would preserve variance. 

{\bf Shared bases for multi-class problems}.
When we are learning a multi-class classifier, it is observed that learning a unique polynomial for each class creates a large model (with $\cO(2drC)$ weights) when the number of classes $C$ is large, which leads to overfitting and issues with regularization. Instead, we propose {\em {\color{Maroon}sharing}} bases ($\bu$) across classes for all higher order terms, and learning {\em{\color{OliveGreen} class-specific}} singular values ($\lambda$) for all higher order terms ($\geq 2$) per class. For any input $\bx$, $P(\bx; \btheta) = \textsc{SoftMax}\{P_c(\bx; \btheta)\}_{c=1}^C$, where, for any class $c \in \{1, ..., C\}$,
\begin{align*}
    P_c(\bx; \btheta) = {\color{OliveGreen}b_c} + \langle {\color{OliveGreen}\bu_1^c}, \bx\rangle + \sum_{i=1}^{r_2} {\color{OliveGreen}\lambda_{2i}^c}\cdot \langle {\color{Maroon}\bu_{2i}}, \tilde\bx_2\rangle^2 + \sum_{i=1}^{r_3} {\color{OliveGreen}\lambda_{3i}^c}\cdot \langle {\color{Maroon}\bu_{3i}}, \tilde\bx_3\rangle^3 ... + \sum_{i=1}^{r_k} {\color{OliveGreen}\lambda_{ki}^c}\cdot \langle {\color{Maroon}\bu_{ki}}, \tilde\bx_k\rangle^k.
\end{align*}
The terms in {\em {\color{OliveGreen} green}} denote weights unique to each class, and terms in {\em {\color{Maroon} red}} denote weights shared across classes. This reduces the model size to $\cO((d+r)C + rd)$ from $\cO(rdC)$. We set $\ell$ as the cross-entropy of the softmax of $\{P_c(\cdot, \btheta)\}_{c=1}^C$, as both these operations preserve convexity.

{\bf Exploring nonlinear input transformations}.
Motivated by non-linear GAMs (e.g.,  GA$^2$M~\citep{lou2013accurate} and NAM~\citep{agarwal2021neural}), we modify our low-rank decomposed algorithm by replacing $\tilde\bx_l$ for each $l$ with feature-wise non-linearities. We learn a non-linear {\bf SPAM} $P(\cdot;\btheta)$ as,
\begin{align*}
    P(\bx; \btheta) =  b + \langle \bu_1, F_1(\bx)\rangle + \sum_{i=1}^{r_2} \lambda_{2i}\cdot \langle \bu_{2i}, F_2(\bx)\rangle^2 + ... + \sum_{i=1}^{r_k} \lambda_{ki}\cdot \langle \bu_{ki}, F_k(\bx)\rangle^k.
\end{align*}
Here, the function $F_i(\bx) = \left[f_{i1}(x_1), f_{i2}(x_2), ..., f_{id}(x_d)\right]$ is a Neural Additive Model (NAM)~\citep{agarwal2021neural}. We denote this model as {\sc SPAM-Neural}. For more details on the NAM we use, please refer to the Appendix Section~\ref{sec:appendix_experiment_details}. Note that the resulting model is still interpretable, as the interaction between features $x_i$ and $x_j$, e.g., will be given by $\left(\sum_{k=1}^{r_2} \lambda_{2k}u_{2ki}u_{2kj}\right)\cdot f_{2i}(x_i) f_{2j}(x_j)$ instead of the previous  $\left(\sum_{k=1}^{r_2} \lambda_{2k}u_{2ki}u_{2kj}\right)\cdot \sqrt{x_i\cdot x_j}$ for typical {\sc SPAM}.

{\bf Dropout for bases via $\lambda$}. To capture non-overlapping feature correlations, we introduce a {\em basis dropout} by setting the contribution from any particular basis to zero at random via the singular values $\lambda$. Specifically, observe that the contribution of any basis-singular value pair $(\bu, \lambda)$ at order $l$ is $\lambda\langle \bu, \bx \rangle^l$. We apply dropout to $\lambda$ to ensure that the network learns robust basis directions.
\subsection{Approximation and Learning-Theoretic Guarantees}
We present some learning-theoretic guarantees for SPAM. Proofs are deferred to Appendix Section~\ref{sec:appendix_full_proofs}.
\begin{proposition}
\label{prop:approximation}
Let $f$ be a continuous real-valued function over a compact set $\cX$. For any threshold $\epsilon > 0$, there exists a low-rank polynomial $P : \cX \rightarrow \bbR$ such that $\sup_{\bx\in\cX}\left| f(\bx) - P(\bx) \right| < \epsilon$.
\end{proposition}
The above result demonstrates that asymptotically, decomposed polynomials are universal function approximators. Next, we present an assumption on the data-generating distribution as well as a novel excess risk bound that characterizes learning with polynomials more precisely (non-asymptotic).
\begin{assumption}[Exponential Spectral Decay of Polynomial Approximators]
\label{assumption:singular_value_decay}
Let $\cP_k$ denote the family of all polynomials of degree at most $k$, and let $P_{\star, k}$ denote the optimal polynomial in $\cP_k$, i.e.,  $P_{\star, k} = \arg\min_{P\in\cP_k} \bbE_{(\bx, y)\sim\fP}[\ell(P(\bx), y)]$. We assume that $\forall k$, $P_{\star, k}$ admits a decomposition as described in Equation~\ref{eq:low_rank_rep} such that, for all $1 \leq l \leq k$, there exist absolute constants $C_1 < 1$ and $C_2 = \cO(1)$ such that $|\lambda_{lj}| \leq C_1\cdot \exp\left(-C_2\cdot j^\gamma\right)$ for each $j \geq 1$ and $l \in [1, k]$.
\end{assumption}
This condition is essentially bounding the ``smoothness'' of the nearest polynomial approximator of $f$, i.e., implying that only a few degrees of freedom suffice to approximate $f$ accurately. Assumption~\ref{assumption:singular_value_decay} provides a soft threshold for singular value decay. One can also replace the exponential with a slower ``polynomial'' decay with similar results, and we discuss this in the Appendix Section~\ref{sec:appendix_full_proofs}. We can now present our error bound for $L_2$ regularized models (see Appendix for $L_1$-regularized models). 
\begin{theorem}
\label{thm:primary_generalization_bound}
Let $\ell$ be 1-Lipschitz, $\delta \in (0, 1]$, the generalization error for the optimal degree-$k$ polynomial $P_{\star, k}$ be $\cE(P_{\star_, k})$ and the generalization error for the empirical risk minimizer SPAM $\widehat P_{\br, k}$, i.e., $\widehat P_{\br, k} = \argmin_{P} \sum_{i=1}^n\ell(P(\bx_i), y_i)$ with ranks $\br = [1, r_2, ..., r_k]$ be  $\cE(\widehat P_{\br, k})$. Then, if $\fP$ satisfies Assumption~\ref{assumption:singular_value_decay} with constants $C_1, C_2$ and $\gamma$, we have for $L_2-$regularized ERM, where $\lVert \bu_{ij} \rVert_2 \leq B_{u, 2}\ \forall i\in [k], j \in [r_i]$, and $\lVert \blambda \rVert_2 \leq B_{\lambda, 2}$ where $\blambda = \{\{\lambda_{ij}\}_{j=1}^{r_i}\}_{i=1}^k$, with probability at least $1-\delta$,
\begin{align*}
    \cE(\widehat P_{\br, k})- \cE(P_{\star, k})  \leq 2B_{\lambda, 2} \left(\sum_{l=1}^k(B_{u, 2})^l\sqrt{r_l}\right)\sqrt{\frac{d}{n}}  + \frac{C_1}{C_2}\cdot\left(\sum_{i=2}^k\exp(-r_i^\gamma)\right) + 5\sqrt{\frac{\log\left(4/\delta\right)}{n}}.
\end{align*}
\end{theorem}
The above result demonstrates an expected scaling of the generalization error, matching the bounds obtained for $L_1$ (see Appendix Section~\ref{sec:appendix_full_proofs}) and $L_2-$regularized linear classifiers~\citep{wainwright2019high}. We that the rank $r = \lVert \br \rVert_1$ has an $\cO(r)$ dependence on the Rademacher complexity. This highlights a trade-off between optimization and generalization: a larger $r$ gives a lower approximation error (smaller second term), but a larger model complexity (larger first term). Observe, however, that even when $r = \Omega((\log d)^p)$ for some $p\geq 1$, the approximation error (second term) diminishes as $o(\frac{1}{d^{p\gamma}})$. This suggests that in practice, one only needs a cumulative rank (poly)logarithmic in $d$. The above result also provides a non-asymptotic variant of Proposition~\ref{prop:approximation}; as $n, r\rightarrow\infty$, $\cE(\widehat P_{\br, k})\rightarrow \cE(P_{\star, k})_+$.
\section{Offline Experiments}
\label{sec:experiments}
Here we evaluate SPAM against a set of benchmark algorithms on both classification and regression problems. The baseline approaches we consider are (see Appendix Section~\ref{sec:appendix_experiment_details} for precise details):
\begin{itemize}[leftmargin=*,itemsep=0pt,topsep=0pt,parsep=0pt,partopsep=0pt]
    \item {\bf Deep Neural Networks (DNN)}: These are standard multi-layer fully-connected neural networks included to demonstrate an upper-bound on performance. These are explained via the perturbation-based LIME~\citep{ribeiro2016should} and SHAP~\citep{lundberg2017unified} methods in Section~\ref{sec:human_evaluations}.
    \item {\bf Linear / Logistic Regression}: These serve as our baseline interpretable models. We learn both $L_1$ and $L_2-$regularized models on unary and pairwise features (${d \choose 2}$ features) with minibatch SGD.
    \item {\bf Gradient Boosted Trees (XGBoost)}: We use the library {\tt xgboost}. This baseline is mainly to compare accuracy, as the number of trees required are typically large and are hence uninterpretable.
    \item {\bf Explainable Boosting Machines (EBMs)}~\citep{lou2013accurate}: EBMs use millions of shallow bagged trees operating on each feature at a time. Note that this approach is not scalable to datasets with many features or multi-class problems. We report scores on the datasets where we successfully trained EBMs without sophisticated engineering, using the {\sf interpretml} library~\citep{nori2019interpretml}.
    \item {\bf Neural Additive Models (NAMs)}~\citep{agarwal2021neural}: These models are neural network extensions of prior EBMs. Note that this method also does not scale to some datasets.
\end{itemize}

{\bf Training Setup}. SPAM is implemented by learning $L_1/L_2$-regularized variants by minibatch SGD implemented in PyTorch. For regression tasks, we measure the root mean squared error (RMSE). For binary classification, we report the area under the ROC (AUROC), for multi-class classification, we report the top-1 accuracy (Acc@1), and finally, for object detection, we report mean Average Precision (mAP). We tune hyperparameters via random sampling approach over a grid. Note that for all experiments, both NAM and {\sc SPAM-Neural} have identical MLP structures to ensure identical approximation power. For definitions of metrics and hyperparameter ranges, see Appendix Section~\ref{sec:appendix_experiment_details}.

\subsection{Measuring Benchmark Performance}
\setlength{\tabcolsep}{1.5pt}
\begin{table}[t]
\caption{Evaluation of SPAM on benchmarks against prior work. ($\uparrow$): higher is better, ($\downarrow$): lower is better, $^{**}$: variance across trials is $<0.001$, $^{*~}$: variance across trials is $<0.005$.  {\color{maroon}\bf Boldface red} denotes the best {\color{maroon}\bf black-box} model, {\color{OliveGreen}\bf boldface green} denotes the best {\color{OliveGreen}\bf interpretable} model, {\bf boldface} denotes where SPAM Order 3 is best. Runs averaged over 10 random trials with optimal hyperparameters.}
\label{tab:tabular_benchmark_comparison}
\centering
\small
\begin{tabular}{cccccccc}
\hline
\hline 
\multirow{3}{*~}{Model} & {\bf Regression} & {\bf Binary} & \multicolumn{4}{c}{\bf Multi-Class Classification} & {\bf Obj. Det.} \\ 
& RMSE ($\downarrow$) & AUROC ($\uparrow$) & \multicolumn{4}{c}{Accuracy ($\uparrow$)}  & mAP ($\uparrow$) \\ \cline{2-8}
& \multicolumn{1}{c}{CH} & \multicolumn{1}{c}{FICO} & \multicolumn{1}{c}{CovType} & \multicolumn{1}{c}{News} & \multicolumn{1}{c}{CUB} & \multicolumn{1}{c}{iNat} & \multicolumn{1}{c}{CO114}\\ \hline
\multicolumn{8}{c}{ {\color{OliveGreen}\bf Interpretable} Baselines} \\
\hline
Linear (Order 1) & $0.7354^{**}$ & $0.7909^{**}$ & $0.7254^{**}$ & $0.8238^{**}$ & $0.7451^{**}$ & $0.3932^{**}$ & $0.1917^{**}$\\
Linear (Order 2) & $0.7293^{**}$ & $0.7910^{**}$ & $0.7601^{**}$ & --- & $0.7617^{**}$ & $0.4292^{**}$ & $0.2190^{**}$\\
EBMs (Order 1) & $0.5586^{**}$ & $0.7985^{**}$& $0.7392^{**}$& --- & --- & --- & --- \\
EB$^2$Ms (Order 2) & $0.4919^{**}$ & $0.7998^{**}$& --- & --- & --- & --- & --- \\
NAM & $0.5721^{*~}$ & $0.7993^{**}$ & $0.7359^{**}$ & --- & $0.7632^{**}$ & $0.4194^{**}$ & $0.2056^{**}$ \\ \hline
\multicolumn{8}{c}{ {\color{maroon}\bf Uninterpretable} Black-Box Baselines} \\ \hline
XGBoost  &{\color{maroon}$\mathbf{0.4428^{**}}$} & $0.7925^{**}$ & $0.8860^{**}$ & $0.7677^{**}$ & $0.7186^{**}$ & --- & --- \\ 
DNN & $0.5014^{**}$ & {\color{maroon}$\mathbf{0.7936^{**}}$} & {\color{maroon}$\mathbf{0.9694^{**}}$} & {\color{maroon}$\mathbf{0.8494^{**}}$} & {\color{maroon}$\mathbf{0.7684^{**}}$} & {\color{maroon}$\mathbf{0.4584^{**}}$} & {\color{maroon}$\mathbf{0.2376^{**}}$}\\ \hline
\multicolumn{8}{c}{ {\color{OliveGreen}\bf Our Interpretable Models} } \\ \hline
SPAM (Linear, Order 2)& $0.6474^{**}$ & $0.7940^{**}$ & $0.7732^{**}$ & {\color{OliveGreen}$\mathbf{0.8472^{**}}$} & {\color{OliveGreen}$\mathbf{0.7786^{**}}$} & $0.4605^{**}$ & {\color{OliveGreen}$\mathbf{0.2361^{**}}$}\\
SPAM (Neural, Order 2)& {\color{OliveGreen}$\mathbf{0.4914^{**}}$} & {\color{OliveGreen}$\mathbf{0.8011^{*~}}$} & {\color{OliveGreen}$\mathbf{0.7770^{**}}$}& --- & $0.7762^{**}$ & {\color{OliveGreen}$\mathbf{0.4689^{**}}$} & ---\\\hline
        
SPAM (Linear, Order 3)& $0.6410^{**}$ & $0.7945^{**}$ & $0.8066^{**}$ & $\mathbf{0.8520^{*~}}$ & $0.7741^{**}$& $0.4684^{**}$ & $\mathbf{0.2368^{*~}}$\\ 
SPAM (Neural, Order 3)& $0.4865^{*~}$&  $\mathbf{0.8024^{*~}}$& $0.8857^{**}$& --- & $0.7753^{**}$ & $\mathbf{0.4722^{**}}$ & ---\\ \hline
\hline
\end{tabular}
\end{table}
We select tasks to explore a variety of settings from regression to multi-class classification, and also explore different dataset scales, from a few hundred samples and tens of features to 100K-scale datasets (both in the number of samples and data dimensionality), while ensuring that the features are interpretable. Our datasets are summarized in Table~\ref{tab:tabular_datasets}. Please see Appendix Section~\ref{sec:appendix_dataset_details} for details. For all datasets with no defined train-val-test split, we use a fixed random sample of 70\% of the data for training, 10\% for validation and 20\% for testing. For the 20 Newsgroups dataset, we split the pre-defined training split 7:1 for training and validation, respectively.
\begin{table}[t]
\setlength{\tabcolsep}{1.5pt}
    \centering
    \small
    \caption{Tabular Datasets\label{tab:tabular_datasets}}
    \begin{tabular}{ccccc}
    \hline \hline
        Name & California Housing (CH) & FICO & Cover Type (CovType) & Newsgroups  \\ \hline
        Source & \citet{pace1997sparse} & \citet{fico_community}  & \citet{blackard1999comparative} & \citet{lang1995newsweeder} \\
        Instances & 20,640 & 10,459 & 581,012 & 18,828 \\
        Features & 8 & 23 & 54 & 146,016 \\
        Classes & - & 2 & 7 & 20 \\
        Feature Type & Numeric & Mixed & Mixed & TF-IDF \\ \hline \hline
    \end{tabular}
\end{table}

In an effort to scale interpretable approaches beyond tabular datasets, we consider benchmark problems using the ``Independent Concept Bottleneck'' framework of~\citet{koh2020concept} for image classification and object detection. We use a convolutional neural network (CNN) {\bf backbone} ResNet-50~\citep{he2016deep} that is trained to predict {\em interpretable concepts} (e.g., parts of object) from images. After training the {\bf backbone}, we extract predicted concepts for all inputs, and learn an interpretable classifier {\bf head} from these concepts. In experiments, the backbone remains identical for all comparisons, and we only compare the {\bf head} classifiers (see Appendix Section~\ref{sec:appendix_concept_bottleneck_implementation} for more details). We select three datasets for evaluation in the concept bottleneck setting:
\begin{enumerate}[leftmargin=*,itemsep=0pt,topsep=0pt,parsep=0pt,partopsep=0pt]
    \item {\bf Caltech-UCSD Birds (CUB-200)}~\citep{wah2011caltech}: This is a fine-grained visual categorization dataset where one has to classify between 200 species of birds. The interpretable concepts are 278 {\em binary} bird attributes, e.g., the shape of the beak, the color of the wings, etc. The dataset has 5,994 training and 5,794 validation images. See Appendix Section~\ref{sec:appendix_cub_details} for more details. 
    \item {\bf iNaturalist ``Birds''}~\citep{van2018inaturalist, van2021benchmarking}: iNaturalist can be thought of as a larger version of the previous dataset. We only select the ``Birds'' super-category of samples, which has 414,000 training and 14,860 validation instances across 1,486 classes. Since the iNaturalist dataset does not contain dense part and attribute annotations, we use the {\em predicted} concepts from the CUB-200 backbone model extracted on all samples. 
    \item {\bf Common Objects Dataset} (CO114, Proprietary): Here we consider, for the first time, a concept bottleneck model for object detection. We construct a dataset involving common household objects (e.g., bed, saucer, tables, etc.) with their bounding box annotations. For each bounding box, we collect 2,618 interpretable annotations (e.g., parts, color, etc.). The dataset has 2,645,488 training and 58,525 validation samples across 115 classes with 2,618 interpretable concepts. We report the mean Average Precision (mAP) metric. For more details please refer to Appendix Section~\ref{sec:appendix_w3ig_details}.
\end{enumerate}

Our results are summarized in Table~\ref{tab:tabular_benchmark_comparison}. There are four main takeaways from the results. First, observe that both {\sc SPAM-Linear} and {\sc SPAM-Neural} comfortably outperform their prior interpretable counterparts on all datasets (e.g., Order 2 {\sc SPAM-Linear} outperforms all linear methods -- even the full rank pairwise model), and {\sc SPAM-Neural} comfortably outperforms all non-linear baselines. Next, observe that in all datasets but CoverType, second degree interactions suffice to match or even outperform DNN performance. We will discuss CoverType in detail in the next paragraph. Thirdly, observe that {\sc SPAM} models are as scalable as the black-box approaches, whereas prior work (e.g., NAM and EBMs) are not. In the case of NAM, we were unable to scale due to the sheer number of parameters required to model large datasets, e.g., Newsgroups ($\approx$900M parameters), and for EBMs, the training time and memory requirement increased dramatically for large datasets. Furthermore, EB$^2$Ms do not even support multi-class problems. Finally, observe that for many problems, we do not require {\em non-linear} feature transformations once we model feature interactions: e.g., iNat and CUB.

{\bf CoverType Dataset.} XGBoost and DNNs perform substantially better on CoverType compared to linear or neural SPAM. On analysis, we found that existing SPAM models underfit on CoverType, and hence we increased the total parameters for {\sc SPAM-Neural} via {\em subnets}, identical to NAMs~\citep{agarwal2021neural}, where each feature is mapped to $s$ non-linear features (as opposed to 1 originally). With $s=8$ subnets, the performance of {\sc SPAM-Neural} (order 3) improves to $0.9405$. In comparison, NAMs with $8$ subnets provides a lower accuracy of $0.7551$.

\subsection{Ablation Studies}
\label{sec:ablation}

{\bf Rank and Degree of Interactions}.
\label{sec:ablation_rank_degree}
\begin{figure}
    \centering
    \includegraphics[width=\textwidth]{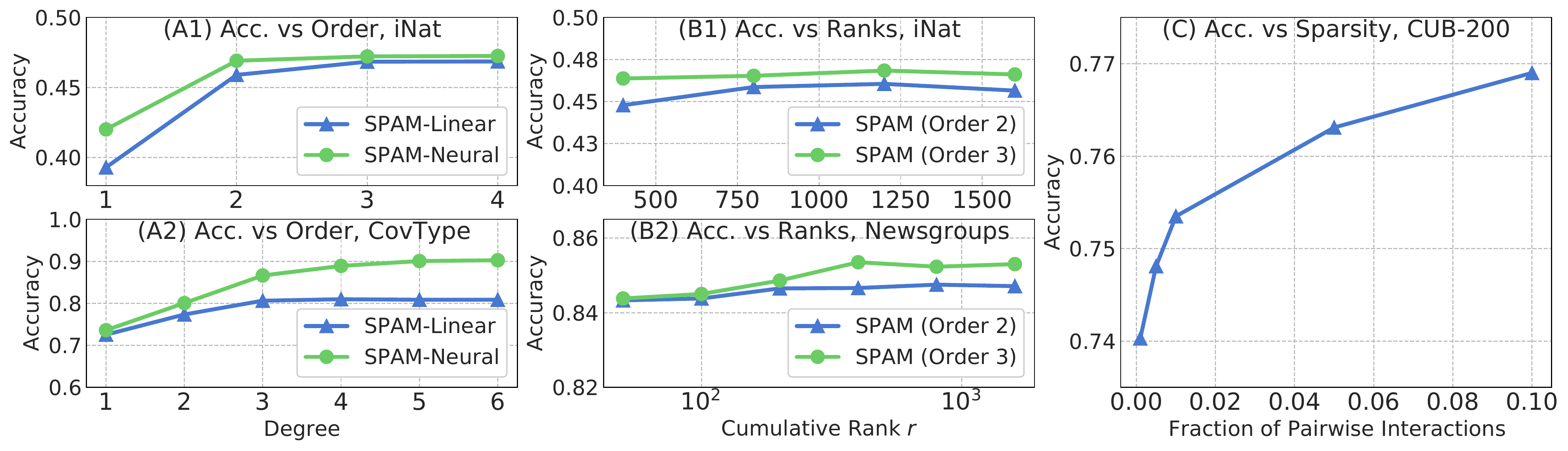}
    \caption{Ablation of Accuracy with: (A1, A2) Degree $k$; (B1, B2) Rank $r$; (C) $L_1$ Sparsity. \label{fig:ablation_ranks_degree}}
\end{figure}
By Proposition~\ref{prop:approximation}, it is natural to expect the approximation quality to increase with the degree $k$ and rank $r$, i.e., as we introduce more higher-order interactions between features. We ablate the degree $k$ on the tabular benchmark {\bf CoverType} and concept bottleneck benchmark {\bf iNaturalist (Birds)}, as summarized in Figure~\ref{fig:ablation_ranks_degree}{\color{maroon}A}. We observe, as expected, that increasing the degree $k$ leads to moderate improvements beyond $k \geq 3$, but with diminishing returns. Similarly, we examine the effect of the cumulative rank $r$ on performance on the {\bf Newsgroups} and {\bf iNaturalist (Birds)} datasets (Figure~\ref{fig:ablation_ranks_degree}{\color{maroon}B}); we observe that performance is sufficiently insensitive to $r$, and plateaus after a while. We imagine that as $r$ increases, model complexity will dominate and performance will likely begin to decrease, matching the full-rank performance at $r = \cO(d^2)$ for pairwise models.
\paragraph{Sparsity in Higher-Order Relationships.}
A requirement for interpretability is to ensure that the learned models can be explained with a few interpretable concepts. Since the number of higher-order combinations increases with $k$, we examine sparsity to limit the number of active feature assignments. We penalize the objective in Equation~\ref{eqn:optimization_fn} with a regularization term that inhibits dense feature selection. If $\bU = \{\{\bu_{li} \}_{i=1}^{r_l}\}_{l=1}^k$ denotes all the basis vectors in matrix form, we add the penalty $\cR(\btheta) \triangleq \lVert \bU \rVert_1$. This ensures that every basis $\bu$ only captures limited interactions between features, and hence the overall complexity of the model is limited. We examine accuracy as a function of the fraction of non-zero pairwise interactions for a degree 2, rank 800 SPAM on CUB-200 in Figure~\ref{fig:ablation_ranks_degree}{\color{maroon}C}, and find that only $6\%$ of the possible interactions suffice to obtain competitive performance.

\paragraph{Examining Assumption~\ref{assumption:singular_value_decay} (Spectral Decay).} The spectral decay assumption presented in Assumption~\ref{assumption:singular_value_decay} is crucial to obtain reasonable generalization bounds (Theorem~\ref{thm:primary_generalization_bound}). To test this assumption, we examine the spectra of a {\sc SPAM-Linear} (Order 2) model for all 200 of CUB-200 classes. These are depicted in Figure~\ref{fig:amt_results}{\color{maroon}A}. Furthermore, we fit an exponential model on the decay itself, and obtain (in line with Assumption~\ref{assumption:singular_value_decay}) $\gamma = 3$, $C_1 = 0.54$ and $C_2 = 0.006$, all in accordance with the assumption.

\subsection{Additional Experiments}
\subsubsection{Comparisons on Common Benchmarks}
In addition to the 7 datasets we considered earlier, we evaluate on 9 further benchmark datasets commonly used in the interpretability literature. We evaluate on the MIMIC2, Credit and COMPAS datasets as presented in~\citet{agarwal2021neural}, and the Click, Epsilon, Higgs, Microsoft, Yahoo and Year datasets from~\citet{chang2021node}. The result of this is summarized in Table~\ref{tab:additional_experiments}. We observe that SPAM models are competitive with prior work, and outperforming prior work on a number of tasks (the best interpretable model is in \textbf{bold}).

\subsubsection{Runtime Evaluation}
We provide a comparison of SPAM runtimes on 4 different datasets to establish their scalability. We consider the California Housing (CH,~\citet{pace1997sparse}, 8 features), FICO HELOC (FICO,~\citet{fico_community}, 23 features), CoverType (CovType,~\citet{blackard1999comparative}, 54 features) and 20 Newsgroups (News,~\citet{lang1995newsweeder}, 146,016 features), and evaluate the throughput of the optimal Linear and Neural SPAM models against Neural Additive Models (NAM,~\citet{agarwal2021neural}) and MLPs. Table~\ref{tab:speed} describes the throughput, where we observe that Linear SPAM is more efficient than MLPs, and NeuralSPAM is orders of magnitude faster than NAM, which does not scale to higher orders.
\begin{table}[t]
\centering\small
\caption{\label{tab:speed}Throughput benchmarking (higher is faster).}
\begin{tabular}{l|c|c|c|c}
\hline \hline
\multirow{2}{*}{\textbf{Method}} & \textbf{CH} & \textbf{FICO} & \textbf{CovType} & \textbf{News} \\ \cline{2-5}
& \multicolumn{4}{c}{Throughput (images / second)} \\ \hline
NAM~\citep{agarwal2021neural} & $5\times10^5$ & $1.2\times10^5$ & $8\times10^4$ & 23 \\
NAM (Order 2) & $1.1\times10^4$ & $6\times10^3$ & $3\times10^3$ & - \\ \hline
LinearSPAM (Order 2) & $6.1\times10^7$ & $6.7\times10^7$ & $6.1\times10^7$ & $2.6\times10^6$ \\
NeuralSPAM (Order 2) & $1.7\times10^5$ & $7.9\times10^3$ & $4.1\times10^3$ & - \\
LinearSPAM (Order 3) & $3.2\times10^7$ & $3.7\times10^7$ & $3.9\times10^7$ & $1.8\times10^5$ \\
NeuralSPAM (Order 3) & $1.1\times10^5$ & $5.3\times10^3$ & $2.6\times10^3$ & - \\ \hline
MLP & $1.3\times10^7$ & $1.3\times10^7$ & $1.3\times10^7$ & $2.2\times10^5$ \\ \hline \hline
\end{tabular}
\end{table}

\begin{table}[t]
\small
\caption{\label{tab:additional_experiments} Comparison against state-of-the-art interpretable neural networks on additional benchmarks.}
\begin{tabular}{lccccccccc}
\hline\hline
\multicolumn{1}{c}{\textbf{Method}} & \begin{tabular}[c]{@{}c@{}}MIMIC2 \\ (AUC)\end{tabular} & \begin{tabular}[c]{@{}c@{}}Credit\\ (AUC)\end{tabular} & \begin{tabular}[c]{@{}c@{}}COMPAS\\ (AUC)\end{tabular} & \begin{tabular}[c]{@{}c@{}}Click\\ (ERR)\end{tabular} & \begin{tabular}[c]{@{}c@{}}Epsilon\\ (ERR)\end{tabular} & \begin{tabular}[c]{@{}c@{}}Higgs \\ (ERR)\end{tabular} & \begin{tabular}[c]{@{}c@{}}Microsoft\\ (MSE)\end{tabular} & \begin{tabular}[c]{@{}c@{}}Yahoo\\ (MSE)\end{tabular} & \begin{tabular}[c]{@{}c@{}}Year \\ (MSE)\end{tabular} \\
 \hline
NAM~\citep{agarwal2021neural} & 0.8539 & 0.9766 & 0.7368 & 0.3447 & 0.1079 & 0.2972 & 0.5824 & 0.6093 & 85.25 \\
NODE-GAM~\citep{chang2021node} & 0.8320 & 0.9810 & 0.7420 & \textbf{0.3342} & 0.1040 & 0.2970 & 0.5821 & 0.6101 & 85.09 \\ \hline
LinearSPAM (Order 2) & 0.8514 & 0.9836 & \textbf{0.7426} & 0.3791 & \textbf{0.1011} & 0.2881 & 0.5710 & 0.5923 & 81.30 \\
NeuralSPAM (Order 2) & \textbf{0.8664} & \textbf{0.9850} & 0.7411 & 0.3348 & 0.1020 & \textbf{0.2750} & \textbf{0.5671} & \textbf{0.5869} & \textbf{79.99} \\ \hline
XGBoost & 0.8430 & 0.9780 & 0.7440 & 0.3334 & 0.1112 & 0.2328 & 0.5544 & 0.5420 & 78.53 \\
\hline\hline
\end{tabular}
\end{table}

\section{Human Subject Evaluations}
\label{sec:human_evaluations}


{\bf Experiment Methodology}. We now evaluate how well explanations from {\bf SPAM} fare \underline{\em in a practical setting} with non-interpretable benchmarks such as black-box models equipped with {\em post-hoc} explanations. Our objective is to mimic a practical setting, where the model user must decide between a fully-interpretable model vs. a black-box classifier with {\em post-hoc} explainability. Our experiment design is a Prediction Task~\citep{hoffman2018metrics, muramatsu2001transparent}, where the objective for the users is to guess what the model predicted given an explanation. The motivation of such a design is to ascertain both the \underline{\em faithfulness} and \underline{\em interpretability} of the model explanations.
\begin{table}[t]
    \centering
    \small
    \caption{Feature importances for interpretable models used in human subject evaulations.\label{tab:importances}}
    \begin{tabular}{cc}
    \hline \hline
        Model & Importances for input $\bx = \{x_1, ..., x_d\}$\\ \hline
        Linear / {\em post-hoc} Linear & $w_i\cdot x_i$, $i \in \{1, ..., d\}$ \\
        {\sc SPAM} (Linear, Order 2) & $u_{1i}\cdot x_i, i \in \{1, ...,d\}$ and $\left(\sum_{l=1}^{r_2} \lambda_{2l}u_{2li}u_{2lj}\right)\cdot \sqrt{x_i\cdot x_j}, i,j \in [d]$ \\ \hline
        \hline
    \end{tabular}
\end{table}

{\bf Computing Explanations}. We first compute the feature importances for any class $c$, as the {\em contribution} of any specific feature in the logit for class $c$ (for binary classification or regression, there is only one class). For the Logisitc Regression model, we simply use the contribution of each feature to the prediction. For SPAM-Linear (Order 2), we still use each feature's contribution, but features now can be unary ($x_i$) or pairwise ($x_i$ and $x_j$). As black-box {\em post-hoc}
baselines, we consider DNNs with LIME~\citep{ribeiro2016should} and KernelSHAP~\citep{lundberg2017unified} to generate {\em post-hoc} linear explanations. See Table~\ref{tab:importances} for the formulation of importance for the models we use in experiments.  

{\bf Experiment Design}
We conduct a separate experiment for each model $\bf M$ with explanations of length $E$ by selecting $E$ most important features. 
Each such experiment runs in two phases - {\em training} and {\em testing}. In the {\em training} phase, participants are shown 8 images of birds and their corresponding explanations to develop their mental model. Four of these images have been predicted as class {\bf A} by the model, and the remaining are predicted as class {\bf B}. 
They then move on to the {\em testing} phase, where they are successively shown the explanations for 5 unseen images (which the model could have predicted as either {\bf A} or {\bf B}), and the users must answer \underline{\em ``given the explanation, which class (of {\bf A} or {\bf B}) did the model predict?''}. If they desire, the users can move back to the training phase at any time to revise model explanations. 
We do not show the corresponding images in the test  phase, as we want the user feedback to rely solely on the faithfulness of explanations. We measure the {\em mean user accuracy} on predicting model decisions. For more details on the interface, please see Appendix Section~\ref{sec:appendix_amt_details}.
For each experiment, we follow an independent-subjects design,
and gathered a total of 150 participants per experiment. Each task (corresponding to one model-explanation pair) lasted 3 minutes. The participants were compensated with $\$$0.75 USD ($\$$15/hour), and all experiments were run on Amazon Mechanical Turk (AMT,~\citet{buhrmester2016amazon}) with a total cost of $\$$3000. To remove poorly performing participants, we only select those that get the first decision correct, and compute the mean user accuracy using the remaining 4 images.

{\bf Study \#1: Comparing Black-Box and Transparent Explanations}. We compare different interpretable models with a {\em fixed} explanation length of $E=7$. 
Our objective is to assess the interpretability of pairwise interactions compared to black-box and linear explanations, and to re-answer whether it was even necessary to have fully-interpretable models compared to black-box models with {\em post-hoc} explanations. We used samples from the 5 pairs of classes, that is {\bf A} and {\bf B}), from CUB-200 dataset (see Section~\ref{sec:appendix_amt_details} for images). The results from this experiment are summarized in Figure~\ref{fig:amt_results}{\color{maroon}B}. Observe that with SPAM, the mean user accuracy (mUA) is substantially higher (0.71) compared to both linear (0.67) and {\em post-hoc} LIME (0.65). We would like to remark that even for $E=7$, SPAM has an average of $3.1$ pairwise terms in each explanation, and therefore the higher-order terms are substantial. A one-sided t-test~\citep{freund1967modern} provides a significance ($p-$value) of $2.29\times 10^{-4}$ and a test statistic of $3.52$ for the SPAM mUA being higher than LIME, and correspondingly for Linear, we get a significance ($p-$value) of $0.03801$ and a test statistic of $1.666$. Hence SPAM's improvements in interpretability are statistically significant over both linear and {\em post-hoc} approaches.

\begin{figure}
    \centering
    \includegraphics[width=\textwidth]{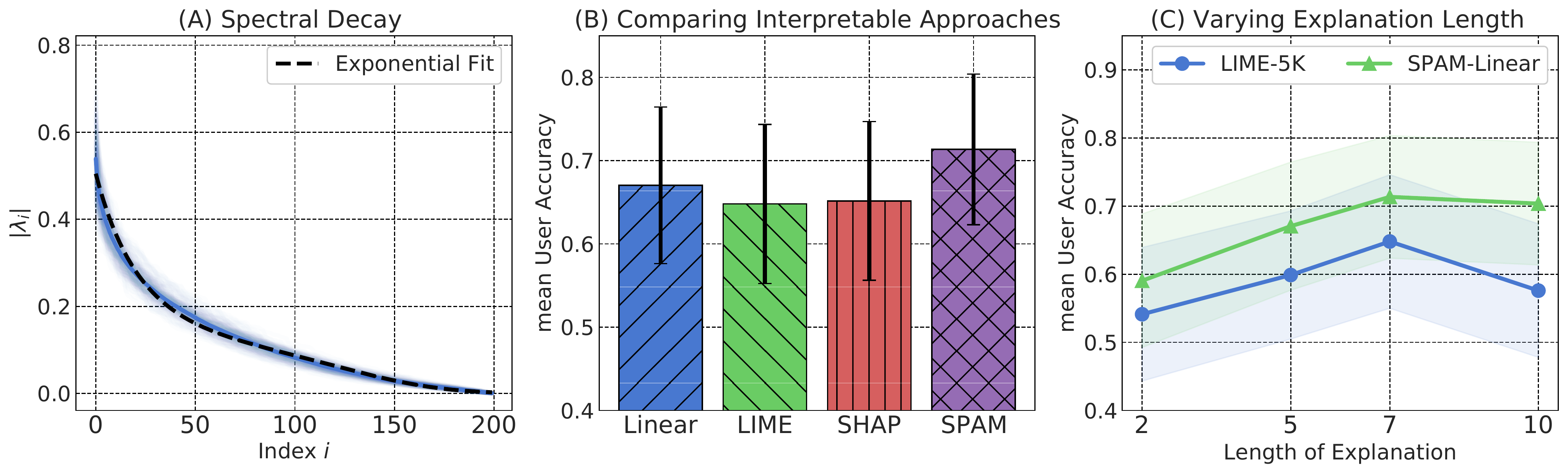}
    \caption{(A) Spectral Decay on CUB-200; Human Subject Evaluation results for (B) Comparing Black-box and glass-box explanations, and (C) LIME vs. SPAM with different explanation lengths.\label{fig:amt_results}}
\end{figure}

{\bf Study \#2: Varying Explanation Length}. One can argue that the increased interpretability of polynomial approaches is due to the fact that each higher-order ``term'' involves multiple features (e.g., pairwise involves two features), and hence the increased mean user accuracy (mUA) is due to the larger number of features shown (either individually or as pairs). We examine this hypothesis by varying the explanation length $E$ for both SPAM-Linear (Order 2) and LIME, which consequently increases the number of terms seen for both approaches. We observe in Figure~\ref{fig:amt_results}{\color{maroon}C} that regardless of explanation method, mUA is maximum at $E=7$, and introducing more terms decreases interpretability.

\section{Discussion and Conclusion}
\label{sec:conclusion}
We presented a simple and scalable approach for modeling higher-order interactions in interpretable machine learning. Our approach alleviates several of the  concerns with existing algorithms and focuses on presenting a viable alternative for practical, large-scale machine learning systems. In addition to offline experimental guarantees, we demonstrate by an extensive user study that these models are indeed more interpretable in practice, and therefore can readily substitute {\em post-hoc} interpretability without any loss in performance. 
Moreover, to the best of our knowledge, our work provides the first incremental analysis of interpretability and performance: we show how progressively increasing the model complexity (with higher order interactions) brings better performance while compromising interpretability in practice, and the ``sweet spot'' appears to rest at pairwise interactions.

We have several follow-up directions. First, given that our model provides precise guarantees on generalization, it is a feasible starting point to understanding tradeoffs between privacy and explainability from a rigorous perspective. Next, one can consider utilizing SPAM to understand failure modes and spurious correlations directly. Furthermore, SPAM-esque decompositions can also be useful in other domains beyond interpretability, e.g., language and vision, where modeling higher-order interactions is challenging due to the curse of dimensionality.

\bibliography{neurips2021}
\bibliographystyle{abbrvnat}
\clearpage

\section*{Checklist}

\begin{enumerate}

\item For all authors...
\begin{enumerate}
  \item Do the main claims made in the abstract and introduction accurately reflect the paper's contributions and scope?
    \answerYes{}
  \item Did you describe the limitations of your work?
    \answerYes{See conclusion.}
  \item Did you discuss any potential negative societal impacts of your work?
    \answerYes{}
  \item Have you read the ethics review guidelines and ensured that your paper conforms to them?
    \answerYes{}
\end{enumerate}

\item If you are including theoretical results...
\begin{enumerate}
  \item Did you state the full set of assumptions of all theoretical results?
    \answerYes{See Appendix Section~\ref{sec:appendix_full_proofs}.}
        \item Did you include complete proofs of all theoretical results?
    \answerYes{See Appendix Section~\ref{sec:appendix_full_proofs}.}
\end{enumerate}

\item If you ran experiments...
\begin{enumerate}
  \item Did you include the code, data, and instructions needed to reproduce the main experimental results (either in the supplemental material or as a URL)?
    \answerYes{}
  \item Did you specify all the training details (e.g., data splits, hyperparameters, how they were chosen)?
    \answerYes{Please see Section~\ref{sec:experiments} and supplementary material.}
        \item Did you report error bars (e.g., with respect to the random seed after running experiments multiple times)?
    \answerYes{}
        \item Did you include the total amount of compute and the type of resources used (e.g., type of GPUs, internal cluster, or cloud provider)?
    \answerYes{}
\end{enumerate}

\item If you are using existing assets (e.g., code, data, models) or curating/releasing new assets...
\begin{enumerate}
  \item If your work uses existing assets, did you cite the creators?
    \answerYes{Please see Sections~\ref{sec:human_evaluations} and~\ref{sec:experiments}.}
  \item Did you mention the license of the assets?
    \answerYes{See accompanying supplementary material.}
  \item Did you include any new assets either in the supplemental material or as a URL?
    \answerNA{No new assets except code, see 3(a).}
  \item Did you discuss whether and how consent was obtained from people whose data you're using/curating?
    \answerYes{}
  \item Did you discuss whether the data you are using/curating contains personally identifiable information or offensive content?
    \answerYes{}
\end{enumerate}

\item If you used crowdsourcing or conducted research with human subjects...
\begin{enumerate}
  \item Did you include the full text of instructions given to participants and screenshots, if applicable?
    \answerYes{See supplementary material.}
  \item Did you describe any potential participant risks, with links to Institutional Review Board (IRB) approvals, if applicable?
    \answerNA{No participant risks were identified in the experiment design.}
  \item Did you include the estimated hourly wage paid to participants and the total amount spent on participant compensation?
    \answerYes{See Section~\ref{sec:human_evaluations}.}
\end{enumerate}

\end{enumerate}
\newpage
\appendix
\section{Deferred Proofs}
\label{sec:appendix_full_proofs}
\subsection{Proof of Proposition~\ref{prop:convexity}}
This follows from noting that the Hessian of the objective function is positive definite for all $\lambda > 0$ when the constraints in the proposition hold.
\subsection{Proof of Proposition~\ref{prop:approximation}}
This follows directly from the Stone-Weierstrass Theorem~\citep{stone1948generalized, weierstrass1885analytic} and Theorem 1 of~\cite{chui1993realization}.
\subsection{Proof of Theorem~\ref{thm:primary_generalization_bound}}
We state Theorem~\ref{thm:primary_generalization_bound} in full here.
\begin{theorem}
\label{thm:appendix_primary_generalization_bound}
Let $\ell$ be a 1-Lipschitz loss, $\delta \in (0, 1]$ and the generalization error for the optimal degree $k$ polynomial $P_{\star, k}$ be given by $\cE(P_{\star_, k})$ and the training error for the ERM SPAM $\widehat P_{\br, k}$, i.e., $\widehat P_{\br, k} = \argmin_{P} \sum_{i=1}^n\ell(P(\bx_i, y_i)$ with ranks $\br = [1, r_2, ..., r_k]$ be given by $\widehat\cE_n(\widehat P_{\br, k})$. Then, for $L_2-$regularized models, where $\lVert \bu_{ij} \rVert_2 \leq B_{u, 2}$ for all $i\in [k], j \in [r_i]$, and $\lVert \blambda \rVert_2 \leq B_{\lambda, 2}$ where $\blambda = \{\{\lambda_{ij}\}_{j=1}^{r_i}\}_{i=1}^k$, we have that with probability at least $1-\delta$ there exists an absolute constant $C$ such that,
\begin{align*}
    \cE(\widehat P_{\br, k})- \cE(P_{\star, k})  \leq 2B_{\lambda, 2} \left(\sum_{l=1}^k(B_{u, 2})^l\sqrt{r_l}\right)\sqrt{\frac{d}{n}}  + C\cdot\left(\sum_{i=2}^k\exp(-r_i^\gamma)\right) + 5\sqrt{\frac{\log\left(4/\delta\right)}{n}}.
\end{align*}
For $L_1-$regularization, $\lVert \bu_{ij} \rVert_1 \leq B_{u, 1}$ for all $i\in [k], j \in [r_i]$, and $\lVert \blambda \rVert_1 \leq B_{\lambda, 1}$ where $\blambda = \{\{\lambda_{ij}\}_{j=1}^{r_i}\}_{i=1}^k$, we have with probability at least $1-\delta$ there exists an absolute constant $C$ such that,
\begin{align*}
    \cE(\widehat P_{\br, k}) - \cE(P_{\star, k}) \leq 2B_{\lambda, 1} \left(\sum_{l=1}^k(B_{u, 1})^l\right)\sqrt{\frac{\log(d)}{n}}  + C\cdot\left(\sum_{i=2}^k\exp(-r_i^\gamma)\right) + 5\sqrt{\frac{\log\left(4/\delta\right)}{n}}.
\end{align*}
\end{theorem}
While we will eventually prove Theorem~\ref{thm:primary_generalization_bound}, the approach is to first establish the necessary mathematical background and outline the proof approach. We will provide a general result under arbitrary Lipschitz loss functions via a metric entropy bound for low-rank polynomial approximations. We will denote, for any $k > 1$, the space of all {\em low-rank} polynomials of order $k$ with cumulative rank $\br = [1, r_1, r_2, ..., r_k] \in \bbR^{k}$ as $\bTheta_\br \subset \bbR_+^{d(r+1)}$ where $r = \lVert \br \rVert_1$. Recall from Section~\ref{sec:method} that any $\btheta \in \bTheta$ is composed of the components, $\{b, \{\{\lambda_{ij}\}_{j=1}^{r_i}\}_{i=1}^k, \{\{\bu_{ij}\}_{j=1}^{r_i}\}_{i=1}^k\}$ and we use $\btheta$ as a shorthand.

Now, for any function $f: \cX \rightarrow \cY$ and any bounded $L-$Lipschitz loss function $\ell : \cY \times \cY \rightarrow [0, 1]$ (we assume, without loss of generality, that the range of $\ell$ is the closed interval $[0,1]$), the {\em training error} over $n$ samples as,
\begin{align*}
    \widehat\cE_n(f) = \frac{1}{n}\sum_{i=1}^n \left[\ell(f(\bx_i), y_i)\right].
\end{align*}
Similarly, we can define the {\em expected risk} over the sample distribution $\fP: \cX \times \cY \rightarrow [0, 1)$ as,
\begin{align*}
    \cE(f) = \bbE_{(\bx, y) \sim \fP}\left[\ell(f(\bx), y)\right].
\end{align*}
At a high level, our objective is to bound the {\em expected risk} of the ERM {\em low-rank decomposed polynomial} (LRDP) $\cE(\widehat{P}_{\br, k})$ of some degree $k$ and (cumulative) rank $r$ with risk of the {\em optimal} polynomial of degree $k$, i.e., $\cE(P_{\star, k})$. Note that $P_{\star, k}$ is not necessarily low-rank and can have cumulative rank of order $\cO(d^k)$. Let us denote the maximum possible rank of any polynomial of degree $k$ as $\bar{r}$ and the corresponding set of ranks as $\bar{\br}$. Observe that any polynomial of degree $k$ has a (non-unique) representation in $\bTheta_{\bar{\br}}$, and furthermore, 
\begin{align}
\label{eqn:appendix_opt_risk}
    P_{\star, k} = \argmin_{\btheta \in \bTheta_{\bar{\br}}} \left(\cE(P(\cdot; \btheta))\right).
\end{align}
Similarly, we have that the {\em empirical risk minimizer} LRDP $\widehat{P}_{\br, k}$ of degree $k$ and ranks $\br$ satisfies,
\begin{align}
\label{eqn:appendix_erm_risk}
    \widehat{P}_{\br, k} = \argmin_{\btheta \in \bTheta_{\br}} \left(\widehat\cE_n(P(\cdot; \btheta))\right).
\end{align}
We can now begin the proof. For the optimal polynomial $P_{\star, k}$, let us denote the corresponding singular values as $\{\{\lambda^\star_{ij}\}_{j=1}^{\bar{r}_i}\}_{i=1}^k$, bases as $\{\{\bu^\star_{ij}\}_{j=1}^{\bar{r}_i}\}_{i=1}^k$ and bias as $b_{\star, k}$. We will bound the {\em excess risk} $\cE(\widehat P_{\br, k}) - \cE(P_{\star, k})$ by introducing a third polynomial $\widetilde P_{\br, k} \in \bTheta_{\br}$ which is defined as a {\em truncated} version of $P_{\star, k}$ up to the ranks $\br$. Specifically, we set the singular values $\{\{\widetilde\lambda_{ij}\}_{j=1}^{{r}_i}\}_{i=1}^k$, bases as $\{\{\widetilde\bu_{ij}\}_{j=1}^{{r}_i}\}_{i=1}^k$ and bias as $\widetilde b_{k}$ as,
\begin{align}
\label{eqn:appendix_truncated_optimal_polynomial}
    \text{For } 1\leq j \leq r_i, 1\leq i \leq k,\ \ \widetilde\lambda_{ij} = \lambda^\star_{ij}, \widetilde\bu_{ij} = \bu^\star_{ij}, \text{ and } \widetilde b_{k} = b^\star_k.
\end{align}
The polynomial $\widetilde P_{\br, k}$ can therefore be thought of as a ``truncated'' version of the optimal polynomial $P_{\star, k}$ up to the ranks $\br$. Observe that we can then write the excess risk as,
\begin{align*}
    \cE(\widehat P_{\br, k}) - \cE(P_{\star, k}) &= \underbrace{\cE(\widehat P_{\br, k}) - \widehat\cE_n(\widehat P_{\br, k})}_{\Circled{A}} + \underbrace{\widehat\cE_n(\widehat P_{\br, k}) - \widehat\cE_n(\widetilde P_{\br, k})}_{\leq 0} + \underbrace{\widehat\cE_n(\widetilde P_{\br, k}) - \cE(P_{\star, k})}_{\Circled{B}}.
\end{align*}
The term $\widehat\cE_n(\widehat P_{\br, k}) - \widehat\cE_n(\widetilde P_{\br, k}) \leq 0$ since $\widehat P_{\br, k}$ minimizes the empirical risk within $\bTheta_\br$ (Equation~\ref{eqn:appendix_erm_risk}). Hence bounding terms $\Circled{A}$ and $\Circled{B}$ will provide us the bound. We will first bound term $\Circled{B}$ via Lemma~\ref{lemma:appendix_term_b}. We have that with probability at least $1-\delta/2$ for any $\delta \in (0, 1]$, 
\begin{align*}
    \left|\widehat\cE_n(\widetilde P_{\br, k}) - \cE(P_{\star, k})\right| \leq LC\cdot\left(\sum_{i=2}^k\exp(-r_i^\gamma)\right) + 2\sqrt{\frac{\log\left(2/\delta\right)}{n}}.
\end{align*}
To bound term $\Circled{A}$, we proceed via bounding the Rademacher complexity, a classic approach in statistical learning~\citep{wainwright2019high}. We will first establish a bound on the Rademacher complexity of $\bTheta_\br$ under $L^1$ and $L^2-$regularization. Note that similar bounds via directly bounding the metric entropy or Dudley's entropy can also be obtained, but we present this approach for simplicity. We anticipate that unless alternative assumptions are made about the polynomials, the bounds will remain identical via these approaches. Observe first, that since the loss function is Lipschitz and bounded, we have with probability at least $1-\delta/2, \delta \in (0, 1]$, from Theorem~\ref{thm:appendix_rademacher_struct} and Theorem 8 of~\citet{bartlett2002rademacher},
\begin{align}
    \cE(\widehat P_{\br, k}) - \widehat\cE_n(\widehat P_{\br, k}) \leq \fR_n(\ell \odot \cP(\bTheta_\br)) + \sqrt{\frac{8\log(4/\delta)}{n}}.
\end{align}
Where $\fR_n$ denotes the empirical Rademacher complexity at $n$ samples~\citep{bartlett2002rademacher}, and $\cP(\bTheta_\br)$ denotes the set of all polynomials represented by the parameterization $\bTheta_\br$. Now, observe that any element in $\cP(\bTheta_\br)$ comprises $k+1$ polynomials of a fixed degree (one for each degree from 0 to $k$). With some abuse of notation, let us denote the family of all polynomials with a degree exactly $l$ and rank exactly $r$ as $\cP_{l, r}$. Observe that by Theorem~\ref{thm:appendix_rademacher_struct},
\begin{align}
    \fR_n(\cP(\bTheta_\br)) \leq \sum_{l=0}^k \fR_n(\cP_{l, r_l}) = \sum_{l=1}^k \fR_n(\cP_{l, r_l}).
\end{align}
The last equality follows from the fact that $\cP_{0, \cdot}$ is always a constant function with a Rademacher complexity of 0. Furthermore, since $\ell$ is Lipschitz, we have once again by Theorem~\ref{thm:appendix_rademacher_struct} that $\fR_n(\ell \odot \cP(\bTheta_\br)) \leq 2L\cdot \fR_n(\cP(\bTheta_\br))$. Putting this together, we have that with probability at least $1-\delta/2, \delta \in (0, 1]$,
\begin{align}
    \cE(\widehat P_{\br, k}) - \widehat\cE_n(\widehat P_{\br, k}) \leq 2L\cdot\sum_{l=1}^k \fR_n(\cP_{l, r_l}) + \sqrt{\frac{8\log(4/\delta)}{n}}.
\end{align}
Now, to bound the remaining term, we consider two regularization constraints for $\bTheta_\br$. The first is $L_2-$regularization, where we assume that the parameter $\lambda, \bu$ are bounded in $L_2$. Let $\blambda = \{\{\lambda_{ij}\}_{j=1}^{r_i}\}_{i=1}^k$ denote the singular values across all ranks. Then, we assume for all $i \in [1, ..., k]$ and $j \in [1, ..., r_i]$ that $\lVert \bu_{ij} \rVert_2 \leq B_{u, 2}$ and $\lVert \blambda \rVert_2 \leq B_{\lambda, 2}$. Under this constraint, we can bound the empirical Rademacher complexity of each polynomial via Lemma~\ref{lem:appendix_empirical_rademacher_l2}. Replacing this result, we obtain with probability at least $1-\delta/2, \delta \in (0, 1]$,
\begin{align}
    \cE(\widehat P_{\br, k}) - \widehat\cE_n(\widehat P_{\br, k}) \leq 2LB_{\lambda, 2} B^x_2 \cdot \sum_{l=1}^k(B_{u, 2})^l\sqrt{\frac{r}{n}} + \sqrt{\frac{8\log(4/\delta)}{n}}.
\end{align}
Where $B^x_2 = \sup_{\bx\in\cX} \lVert \bx \rVert_2$. Without loss of generality we can assume $B^x_2 = \sqrt{d}$. Now, for $L_1-$regularized models, we assume that for all $i \in [1, ..., k]$ and $j \in [1, ..., r_i]$ that $\lVert \bu_{ij} \rVert_1 \leq B_{u, 1}$ and $\lVert \blambda \rVert_1 \leq B_{\lambda, 1}$. This gives us the following bound via Lemma~\ref{lem:appendix_empirical_rademacher_l1}:
\begin{align}
    \cE(\widehat P_{\br, k}) - \widehat\cE_n(\widehat P_{\br, k}) \leq 2LB_{\lambda, 1} B^x_\infty \cdot\sum_{l=1}^k(B_{u, 1})^l\cdot\sqrt{\frac{\log(d)}{n}} + \sqrt{\frac{8\log(4/\delta)}{n}}.
\end{align}
Where $B^x_\infty = \sup_{\bx\in\cX} \lVert \bx \rVert_\infty$. Without loss of generality we can assume $B^x_\infty = 1$. Using the bound for term $\Circled{B}$ and applying a union bound provides us the final result.
\subsection{Polynomial Decay of Spectrum}
As discussed earlier, one can replace the ``exponential decay'' assumption in Assumption~\ref{assumption:singular_value_decay} to a ``softer'' decay criterion with fatter tails, e.g., like a {\em polynomial} decay. We present extensions of earlier results to this polynomial decay condition. First we state the polynomial decay criterion precisely.
\begin{assumption}
\label{assumption:appendix_singular_value_decay_polynomial}
Let $\cP_k$ denote the family of all polynomials of degree at most $k$, and let $P_{\star, k}$ denote the optimal polynomial in $\cP_k$, i.e.,  $P_{\star, k} = \arg\min_{P\in\cP_k} \bbE_{(\bx, y)\sim\fP}[\ell(P(\bx), y)]$. We assume that $\forall k$, $P_{\star, k}$ admits a decomposition as described in Equation~\ref{eq:low_rank_rep} such that, for all $1 \leq l \leq k$, there exist constants $C_1 < 1$ and $C_2 = \cO(1)$ such that $|\lambda_{lj}| \leq C_1\cdot j^{-\gamma-1}$ for each $j \geq 1$ and $l \in [1, k]$.
\end{assumption}
\begin{theorem}
\label{thm:appendix_primary_generalization_bound_polynomial}
Let $\ell$ be a 1-Lipschitz loss, $\delta \in (0, 1]$ and the generalization error for the optimal degree $k$ polynomial $P_{\star, k}$ be given by $\cE(P_{\star_, k})$ and the training error for the ERM SPAM $\widehat P_{\br, k}$, i.e., $\widehat P_{\br, k} = \argmin_{P} \sum_{i=1}^n\ell(P(\bx_i, y_i)$ with ranks $\br = [1, r_2, ..., r_k]$ be given by $\widehat\cE_n(\widehat P_{\br, k})$. Then, for $L_2-$regularized models, where $\lVert \bu_{ij} \rVert_2 \leq B_{u, 2}$ for all $i\in [k], j \in [r_i]$, and $\lVert \blambda \rVert_2 \leq B_{\lambda, 2}$ where $\blambda = \{\{\lambda_{ij}\}_{j=1}^{r_i}\}_{i=1}^k$, we have that with probability at least $1-\delta$ there exists an absolute constant $C$ such that,
\begin{align*}
    \cE(\widehat P_{\br, k})- \cE(P_{\star, k})  \leq 2B_{\lambda, 2} \left(\sum_{l=1}^k(B_{u, 2})^l\sqrt{r_l}\right)\sqrt{\frac{d}{n}}  + C\cdot\left(\sum_{i=2}^k r_i^{-\gamma}\right) + 5\sqrt{\frac{\log\left(4/\delta\right)}{n}}.
\end{align*}
For $L_1-$regularization, $\lVert \bu_{ij} \rVert_1 \leq B_{u, 1}$ for all $i\in [k], j \in [r_i]$, and $\lVert \blambda \rVert_1 \leq B_{\lambda, 1}$ where $\blambda = \{\{\lambda_{ij}\}_{j=1}^{r_i}\}_{i=1}^k$, we have with probability at least $1-\delta$ there exists an absolute constant $C$ such that,
\begin{align*}
    \cE(\widehat P_{\br, k}) - \cE(P_{\star, k}) \leq 2B_{\lambda, 1} \left(\sum_{l=1}^k(B_{u, 1})^l\right)\sqrt{\frac{\log(d)}{n}}  + C\cdot\left(\sum_{i=2}^k r_i^{-\gamma}\right) + 5\sqrt{\frac{\log\left(4/\delta\right)}{n}}.
\end{align*}
\end{theorem}
\begin{proof}
The proof is identical to the proof of Theorem~\ref{thm:primary_generalization_bound}, except for the application of Lemma~\ref{lemma:appendix_term_b}. Note that following Lemma~\ref{lemma:appendix_term_b}, the sum of the singular values can be bound using Assumption~\ref{assumption:appendix_singular_value_decay_polynomial} as,
\begin{align*}
    \sum_{i=1}^k\sum_{j=r_i}^{\bar{r}_i} |\lambda^\star_{ij}| \leq \sum_{i=1}^k\sum_{j=r_i}^{\bar{r}_i} C_1\cdot j^{-\gamma-1} \leq \sum_{i=1}^k\int_{j=r_i}^{\infty} C_1\cdot j^{-\gamma-1}d\gamma = \frac{C_1}{\gamma +1}\sum_{i=1}^kr_i^{-\gamma}.
\end{align*}
Replacing this result completes the proof.
\end{proof}

{\bf Discussion}. Observe that in contrast to the result for exponential decay (where one required $r$ to be polylogarithmic in $d$), we now require a larger rank $r$ to achieve competitive performance. Specifically, one can observe from Theorem~\ref{thm:appendix_primary_generalization_bound_polynomial} that setting $r_i = r/k = d^{\gamma/m}$ for some $m > 1$ guarantees that the error term decays at the rate $\cO\left(k^{\frac{1-\gamma}{m}}d^{-\frac{1}{m}}\right)$ which diminishes quickly for small $m$ and large $d$. However, this degrades the rate of growth of the Rademacher complexity (first term) from a $\cO(\sqrt{d})$ dependence to $\cO(d^{\frac{1}{p}\left(1+\frac{\gamma}{m}\right)})$ for $L_p$ regularization, $p \in \{1, 2\}$.
\subsection{Omitted Results}
\begin{lemma}
\label{lemma:appendix_term_b}
For the polynomials $\widetilde P_{\br, k}$ and $P_{\star, k}$ as defined in Equations~\ref{eqn:appendix_opt_risk} and~\ref{eqn:appendix_truncated_optimal_polynomial}, the following holds with probability at least $1-\delta, \delta \in (0, 1]$, for some absolute constant $C \ll 1$,
\begin{align*}
    \left|\widehat\cE_n(\widetilde P_{\br, k}) - \cE(P_{\star, k})\right| \leq LC\cdot\left(\sum_{i=2}^k\exp(-r_i^\gamma)\right) + 2\sqrt{\frac{\log\left(2/\delta\right)}{n}}.
\end{align*}
\end{lemma}
\begin{proof}
Observe,
\begin{align*}
    \widehat\cE_n(\widetilde P_{\br, k}) - \cE(P_{\star, k}) &= \widehat\cE_n(\widetilde P_{\br, k}) - \cE(\widetilde P_{\br, k}) + \cE(\widetilde P_{\br, k}) - \cE(P_{\star, k}) \\
    &\leq \underbrace{\left| \widehat\cE_n(\widetilde P_{\br, k}) - \cE(\widetilde P_{\br, k}) \right|}_{\Circled{B1}} + \underbrace{\left| \cE(\widetilde P_{\br, k}) - \cE(P_{\star, k}) \right|}_{\Circled{B2}}.
\end{align*}
To bound $\Circled{B1}$, observe that for any point $(\bx, y)$ within the training set, $\bbE[\ell(\widetilde P_{\br, k}(\bx), y)] = \cE(\widetilde P_{\br, k})$. Furthermore, $0 \leq \ell(\cdot, \cdot) \leq 1$. We can therefore apply the Azuma-Hoeffding inequality~\citep{bercu2015concentration} and obtain with probability at least $1-\delta, \delta \in (0, 1]$,
\begin{align*}
    \left| \widehat\cE_n(\widetilde P_{\br, k}) - \cE(\widetilde P_{\br, k}) \right| \leq 2\sqrt{\frac{\log\left(2/\delta\right)}{n}}.
\end{align*}
To bound $\Circled{B2}$, observe that since $\ell$ is $L-$Lipschitz, for some $x_1, x_2, y \in \cY$,
\begin{align*}
    \left|\ell(x_1, y) - \ell(x_2, y)\right| &\leq \left|L\cdot\left|x_1 -y \right| - L\cdot\left|x_2 - y\right|\right| \\
    &\leq L\cdot\left|x_1 - x_2\right|.
\end{align*}
Using this, we have that,
\begin{align*}
    \left| \cE(\widetilde P_{\br, k}) - \cE(P_{\star, k}) \right| &\leq \left|\bbE_{(\bx, y)\sim\fP}\left[ \ell(\widetilde P_{\br, k}(\bx), y) - \ell(P_{\star, k}(\bx), y)\right] \right| \\
    &\leq \bbE_{(\bx, y)\sim\fP}\left[ \left|\ell(\widetilde P_{\br, k}(\bx), y) - \ell(P_{\star, k}(\bx), y)\right|\right] \\
    &\leq L\cdot\bbE_{(\bx, y)\sim\fP}\left[ \left|\widetilde P_{\br, k}(\bx) - P_{\star, k}(\bx)\right|\right] \\
    &\leq L\cdot\sup_{\bx\in\cX}\left|\widetilde P_{\br, k}(\bx) - P_{\star, k}(\bx)\right|.
\end{align*}
Observe now that for any $\bx\in\cX$,
\begin{align*}
    \left|P_{\star, k}(\bx) - \widetilde P_{\br, k}(\bx)\right| = \left|\sum_{i=1}^k\sum_{j=r_i}^{\bar{r}_i}\lambda^\star_{ij}\cdot\langle\bu^\star_{ij}, \tilde\bx_i\rangle^i\right| \leq \sum_{i=1}^k\sum_{j=r_i}^{\bar{r}_i} \left|\lambda^\star_{ij}\cdot\langle\bu^\star_{ij}, \tilde\bx_i\rangle^i\right|\leq \sum_{i=1}^k\sum_{j=r_i}^{\bar{r}_i} |\lambda^\star_{ij}|.
\end{align*}
Recall that Assumption~\ref{assumption:singular_value_decay} can control $\sum_{i=1}^k\sum_{j=r_i}^{\bar{r}_i} |\lambda^\star_{ij}|$. If the spectrum obeys the $D-$finite condition and $r_i = D \forall i\in[1,k]$ we have that $\sum_{i=1}^k\sum_{j=r_i}^{\bar{r}_i} |\lambda^\star_{ij}| = 0$. Next, if the spectrum obeys the $\gamma-$exponential decay, we have that $\lambda_{ij} = C_1\exp(-C_2j^\gamma)$. This gives us,
\begin{align*}
    \sum_{i=1}^k\sum_{j=r_i}^{\bar{r}_i} |\lambda^\star_{ij}| \leq \sum_{i=1}^k\sum_{j=r_i}^{\bar{r}_i} C_1\exp(-C_2j^\gamma) \leq \sum_{i=1}^k\int_{j=r_i}^{\infty} C_1\exp(-C_2j^\gamma).
\end{align*}
By Equation E.16 from~\citet{yang2020function} we can bound the R.H.S. since $\gamma \geq 1$,
\begin{align*}
    \int_{j=r_i}^{\infty} C_1\exp(-C_2j^\gamma) &\leq \frac{C_1}{C_2}\exp\left(-r_i^\gamma\right).
\end{align*}
Replacing this result above completes the proof.
\end{proof}

\begin{lemma}[Empirical Rademacher Complexity under $L_2$ regularization]
\label{lem:appendix_empirical_rademacher_l2}
Let $\cP_{l, r}$ denote the set of polynomials that have rank exactly $r$ and degree exactly $l$ such that they admit a rank-decomposed representation of $\sum_{j=1}^r \lambda_j \cdot \bu_j \otimes ... \otimes \bu_j$, where $\blambda = [\lambda_1, ..., \lambda_r]$ s.t. $\lVert \blambda \rVert_2 \leq B_{\lambda, 2}$ and $\lVert \bu_j \rVert_2 \leq B_{u, 2}$ for all $j \in 1, ..., l$. Let the rescaled data distribution as defined in Equation~\ref{eqn:rescaling} for degree $l$ be given by $\widetilde\cX_l = \{\text{sign}(\bx)\cdot|\bx|^{1/l} | \bx \in \cX\}$ and we sample $n$ points i.i.d. from $\widetilde\cX_l$. Then the empirical Rademacher complexity $\fR_n$ of $\cP_{l, r}$ obeys,
\begin{align*}
    \fR_n(\cP_{l, r}) \leq B_{\lambda, 2}  B^x_2(B_{u, 2})^l\sqrt{\frac{r}{n}}.
\end{align*}
\end{lemma}
\begin{proof}
\begin{align*}
    \fR_n(\cP_{l, r}) &= \frac{1}{n}\bbE\left[\sup_{\blambda \in \bLambda, \cU}\sum_{i=1}^n\epsilon_i \sum_{j=1}^r \lambda_r\cdot \left\langle \bu_j, \tilde\bx_{li}\right\rangle^l\right] \\
    &\leq \frac{1}{n}\sup_{\rVert\blambda\rVert_2 \leq B_{\lambda, 2}}\sum_{j=1}^r \lambda_r\cdot\bbE\left[\sup_{\lVert \bu_j \rVert_2 \leq B_{u, 2}}\sum_{i=1}^n\epsilon_i \cdot\left\langle \bu_j, \tilde\bx_{li}\right\rangle^l\right]  \\
    &\leq \frac{1}{n}\sup_{\rVert\blambda\rVert_2 \leq B_{\lambda, 2}} \lVert\blambda\rVert_2\cdot\sqrt{\sum_{j=1}^r\bbE\left[\sup_{\lVert \bu_j \rVert_2 \leq B_{u, 2}}\sum_{i=1}^n\epsilon_i \cdot\left\langle \bu_j, \tilde\bx_{li}\right\rangle^l\right]^2}  \\
    &\leq \frac{B_{\lambda, 2}}{n}\sqrt{\sum_{j=1}^r\bbE\left[\sup_{\lVert \bu_j \rVert_2 \leq B_{u, 2}}\sum_{i=1}^n\epsilon^2_n \cdot\left\langle \bu_j, \tilde\bx_{li}\right\rangle^{2l}\right]}  \\
    &\leq \frac{B_{\lambda, 2}}{n}\sqrt{\sum_{j=1}^r\bbE\left[\sup_{\lVert \bu_j \rVert_2 \leq B_{u, 2}}\sum_{i=1}^n\lVert \bu_j\rVert_2^{2l}\lVert\tilde\bx_{li}\rVert^{2l}\right]}  \\
    &\leq B_{\lambda, 2} (B_{u, 2})^l B^x_2\sqrt{\frac{r}{n}}.
\end{align*}
\end{proof}

\begin{lemma}[Empirical Rademacher Complexity under $L_1$ regularization]
\label{lem:appendix_empirical_rademacher_l1}
Let $\cP_{l, r}$ denote the set of polynomials that have rank exactly $r$ and degree exactly $l$ such that they admit a rank-decomposed representation of $\sum_{j=1}^r \lambda_j \cdot \bu_j \otimes ... \otimes \bu_j$, where $\blambda = [\lambda_1, ..., \lambda_r]$ s.t. $\lVert \blambda \rVert_1 \leq B_{\lambda, 1}$ and $\lVert \bu_j \rVert_1 \leq B_{u, 1}$ for all $j \in 1, ..., l$. Let the rescaled data distribution as defined in Equation~\ref{eqn:rescaling} for degree $l$ be given by $\widetilde\cX_l = \{\text{sign}(\bx)\cdot|\bx|^{1/l} | \bx \in \cX\}$ and we sample $n$ points i.i.d. from $\widetilde\cX_l$. Then the empirical Rademacher complexity $\fR_n$ of $\cP_{l, r}$ obeys,
\begin{align*}
    \fR_n(\cP_{l, r}) \leq (B_{\lambda, 1}\cdot r) (B_{u, 1})^l B^x_1\sqrt{\frac{\log(d)}{n}}.
\end{align*}
\end{lemma}
\begin{proof}
\begin{align*}
    \fR_n(\cP_{l, r}) &= \frac{1}{n}\bbE\left[\sup_{\blambda, \bu}\sum_{i=1}^n\epsilon_i \sum_{j=1}^r \lambda_r\cdot \left\langle \bu_j, \tilde\bx_{li}\right\rangle^l\right] \\
    &\leq \frac{1}{n}\sup_{\lVert\blambda\rVert_1 \leq B_{\lambda, 1}}\sum_{j=1}^r \lambda_r\cdot\bbE\left[\sup_{\lVert \bu_j \rVert_1 \leq B_{u, 1}}\sum_{i=1}^n\epsilon_i \cdot\left\langle \bu_j, \tilde\bx_{li}\right\rangle^l\right]  \\
    &\leq \frac{1}{n}\sup_{\lVert\blambda\rVert_1 \leq B_{\lambda, 1}} \lVert\blambda\rVert_1\cdot\max_{j\in[r]}\left|\bbE\left[\sup_{\lVert \bu_j \rVert_1 \leq B_{u, 1}}\sum_{i=1}^n\epsilon_i \cdot\left\langle \bu_j, \tilde\bx_{li}\right\rangle^l\right]\right|  \\
    &\leq \frac{B_{\lambda, 1}}{n}\max_{j\in[r]}\left[\bbE\left[\left|\sup_{\lVert \bu_j \rVert_1 \leq B_{u, 1}}\sum_{i=1}^n\epsilon_i \cdot\left\langle \bu_j, \tilde\bx_{li}\right\rangle^{l}\right|\right]\right]  \\
    &\leq \frac{B_{\lambda, 1}}{n}\bbE\left[\max_{j\in[r]}\left|\sup_{\lVert \bu_j \rVert_1 \leq B_{u, 1}}\sum_{i=1}^n\epsilon_i \cdot\sum_{j=1}^r\left\langle \bu_j, \tilde\bx_{li}\right\rangle^{l}\right|\right]  \\
    &= \frac{B_{\lambda, 1}}{n}\bbE\left[\left|\sup_{\lVert \bu_j \rVert_1 \leq B_{u, 1}}\sum_{i=1}^n\max_{j\in[r]}\lVert \bu_j \rVert_1^l \lVert \epsilon_i \tilde\bx_{li} \rVert_\infty^l\right|\right]  \\
    &\leq \frac{(B_{u, 1})^l \cdot B_{\lambda, 1}}{n}\bbE\left[\sum_{i=1}^n \lVert \epsilon_i \bx_{i} \rVert_\infty\right]  \\
    &\leq B_{\lambda, 1} (B_{u, 1})^l B^x_\infty\cdot\sqrt{\frac{\log(d)}{n}}.
\end{align*}
The last line follows from Massart's Finite Lemma~(\citet{massart2000some}, Lemma 5.2).
\end{proof}
\begin{theorem}[Structural Results for Rademacher Complexity, Theorem 12 of~\citet{bartlett2002rademacher}]
\label{thm:appendix_rademacher_struct}
Let $\cF, \cF_1, ..., \cF_k$ and $\cH$ be classes of functions. Then,
\begin{enumerate}
    \item If $\cF \subseteq \cH$, $\fR_n(\cF) \leq \fR_n(\cH)$.
    \item $\fR_n(\sum_{i=1}^k \cF_i)\leq \sum_{i=1}^k \fR_n(\cF_i)$.
    \item For any $L$-Lipschitz function $h$, $\fR_n(\cF \odot h) \leq 2L\cdot\fR_n(\cF)$.
\end{enumerate}
\end{theorem}
\newpage
\section{Related Work}
\label{sec:appendix_related_work}

{\bf Post-hoc explainability}. In contrast to designing fully-interpretable models such as GAMs (Generalized Additive Models), a dominant line of work in interpretable machine learning is of {\em post-hoc} explainability. {\em Post-hoc} methods refers to delivering explanations of model predictions after the prediction has been made. The most popular line of work in this domain is that of {\em instance-based} explanations, i.e., explaining each prediction via a local input-specific explanation. For example, the work of~\citet{ribeiro2016should} introduces LIME (Local Interpretable Model Explanations) that fit a weighted linear regression model, for any input $\bx$, using random samples generated from $\cB_2(\bx)$, where the weights are computed via the distance of the random point from the true sample. SHAP~\citep{lundberg2017unified} computes a similar linear explanation based on Shapley values using influence functions derived from cooperative game theory. It has been shown to unify several prior approaches, e.g., DeepLIFT~\citep{shrikumar2017learning, shrikumar2016not}, LIME~\citep{ribeiro2016should}, and Layerwise Relevance Propagation (LRP)~\citet{bach2015pixel}. Several follow-up works have addressed shortcomings in this line of work, please see~\citet{madsen2021post, chakraborty2017interpretability, du2019techniques, carvalho2019machine} for excellent surveys on recent developments in this area. More importantly, however, these approaches are not effective for high-stakes decision-making. They are notoriously unstable~\citep{ghorbani2019interpretation, lakkaraju2020fool}, expensive to compute~\citep{slack2021reliable}, unjustified~\citep{laugel2019dangers} and in many cases, inaccurate~\citep{lipton2018mythos}.~\citet{rudin2019stop} outlines several of these shortcomings in incredible detail, and hence makes the case to replace {\em post-hoc} interpretable approaches with methods that are inherently explainable.

{\bf Transparent and Interpretable Machine Learning}. Early work has focused on greedy or ensemble approaches to modeling interactions~\citep{friedman2001greedy, friedman2008predictive} that enumerate pairwise interactions and learn additive interaction effects. Such approaches often pick up spurious interactions when data is sparse~\citep{lou2013accurate} and are impossible to scale to modern-sized datasets due to enumeration of individual combinations. 
As an improvement,~\citet{lou2013accurate} proposed 
{\sc GA$^2$M} that uses a statistical test to filter out only ``true'' interactions. 
However {\sc GA$^2$M} fails to scale to large datasets as it requires constant re-training of the model and ad-hoc operations such as discretization of features which may require tuning for datasets with a large dimensionality. Other generalized additive models require expensive training of numerous decision trees, kernel machines or splines~\citep{hastie2017generalized}, which make them unattractive compared to black-box models. 

An alternate approach is 
is to learn interpretable neural network transformations. Neural Additive Models (NAMs,~\citet{agarwal2021neural}) learn a DNN per feature.
TabNet~\citep{arik2021tabnet} and NIT~\citep{tsang2018neural} alternatively modify NN architectures to increase their interpretability. NODE-GAM~\citep{chang2021node} improves NAMs with oblivious decision trees for better performance while maintaining interpretability. 
Our approach is notably distinct from these prior works: we do not require iterative re-training; we can learn {\em all} pairwise interactions regardless of dimensionality; we can train SPAM via backpropagation; and we scale effortlessly to very large-scale datasets.


{\bf Learning Polynomials}. The idea of decomposing polynomials was of interest prior to the deep learning era. Specifically, the work of~\citet{ivakhnenko1971polynomial, oh2003polynomial, 155142} study learning neural networks with polynomial interactions, also known as {\em ridge polynomial networks} (RPNs). However, RPNs are typically not interpretable: they learn interactions of a very high order, and include non-linear transformation. 
Similar rank decompositions have been studied in the context of matrix completion~\citep{recht2011simpler}, and are also a subject of interest in tensor decompositions~\citep{nie2017generating, brachat2010symmetric}, where, contrary to our work, the objective is to decompose existing tensors rather than directly learn decompositions from gradient descent. Recently, ~\citet{chrysos2019polygan, chrysos2020p} use tensor decompositions to learn higher-order polynomial relationships in intermediate layers of generative models.
However, their work uses a recursive formulation and learns high-degree polynomials directly from uninterpretable input data (e.g., images), and hence is non-interpretable.
\newpage
\section{Experimental Details}
\label{sec:appendix_experiment_details}
All our benchmarks are implemented in PyTorch and run on a cluster, where each machine was equipped with $8\times$V100 NVIDIA GPU machines with $32$GB VRAM. We set the batch size of at most $1024$ per GPU, and in case a model is too large to fit a batch size of $1024$, we perform a binary search to find the largest batch size that it will accommodate. All datasets have features renormalized between the range $[0,1]$ for uniformity and simplicity. For all experiments we use the Adam with decoupled weight decay (AdamW) optimizer~\citep{loshchilov2017decoupled}.

\vspace{0.3em}
\subsection{Tabular Dataset Details}
\label{sec:appendix_dataset_details}
We use the following tabular datasets.
\begin{enumerate}[leftmargin=*,itemsep=0pt,topsep=0pt,parsep=0pt,partopsep=0pt]
    \item {\bf California Housing} (CH,~\citet{pace1997sparse}): This dataset was derived from the 1990 U.S. census, and is a regression task, fitting the median house value for California districts using demographic information. The dataset contains 20,640 instances and 8 numeric features. 
    \item {\bf FICO HELOC} (FICO,~\citet{fico_community}): This is a binary task, part of the FICO Explainability challenge. The target variable denotes the ``risk'' of home equity line of credit (HELOC) applicants and features correspond to credit report information. It has 10,459 instances and 23 features.
    \item {\bf Cover Type} (CovType,~\citet{blackard1999comparative}): This multi-class dataset is part of the UCI ML repository~\citep{asuncion2007uci}. The task is classification of forest cover type from cartographic variables. It contains 581,012 instances and 54 attributes over 6 classes.
    \item {\bf 20 Newsgroups} (Newsgroups,~\citet{lang1995newsweeder}): This is a popular benchmark dataset where the problem is multi-class classification of email text into 20 subject categories. Following standard convention, we extract TF-IDF features on the training set, which gives us 18,828 instances and 146,016 features over 20 classes (see Appendix Section~\ref{sec:appendix_newsgroups_dataset} for details).
\end{enumerate}

\vspace{0.3em}
\subsection{20 Newsgroups Feature Extraction}
\label{sec:appendix_newsgroups_dataset}
We use the following feature extraction pipeline from the {\tt sklearn} repository. We first use the {\tt TfIdfVectorizer} to compute the TF-IDF features on only the training set. Next, we transform the testing and validation sets using the learnt vectorizer, and rescale all the features within the range $[0, 1]$. This gives us 18,828 instances and 146,016 features over 20 classes.

\vspace{0.3em}
\subsection{Concept Bottleneck Implementation}
\label{sec:appendix_concept_bottleneck_implementation}
In an effort to scale interpretable approaches to problems beyond tabular datasets, we consider benchmark problems using the ``Sequential Concept Bottleneck'' framework of~\citet{koh2020concept}. The general idea of concept bottleneck (CB) models are to enable feature extraction necessary for problems in domains such as computer vision by learning a black-box model that can first predict human interpretable ``concepts'', and then learn a simple, typically fully-transparent model on these concepts to predict the final category. For example, in the problem of bird species classification from images~\citep{wah2011caltech}, one can consider visible parts and attributes of the birds as intermediary ``concepts'' on which one can learn a linear classifier in order to provide reasonable classification performance while at the same time providing interpretable decisions. Similar approaches have also been explored in the work of, e.g., ~\citet{chen2019looks, ghorbani2019towards} and~\citet{zhang2020invertible}. We consider concept bottleneck models for computer vision tasks. 

\vspace{0.3em}
\subsubsection{CUB-200 Concept Bottleneck}
\label{sec:appendix_cub_details}
The Caltech-UCSD Birds (CUB)-200~\citep{wah2011caltech} dataset is an image classification dataset that has images belonging to 200 different species of birds, a problem that is an instance of the broader problem of {\em fine-grained visual classification}. The dataset is equipped with keypoint annotations of 15 bird parts, e.g., crown, wings, etc., and each part has $\geq 1$ part-attribute labels, e.g., {\em brown leg}, {\em buff neck},  etc. For the concept bottleneck model, we train a convolutional neural network architecture that predicts these intermediary concepts from the images, and then used this trained CNN model (with weights fixed) to extract features that are then trained using one of the several benchmark approaches (including SPAM). One key difference between prior implementations and ours in this concept bottleneck setting is that we ignore laterality within the annotations, e.g., ``left wing'' and ``right wing'' are treated simply as ``wing''.

{\bf CNN architecture}. We train a ResNet-50~\citep{he2016deep} on images resized to the size $448\times 448$ resolution until the last pooling layer of the network ({\tt pool5}). The features are 2048-dimensional but ignoring the spatial pooling, are evaluated on each $14\times 14$ patch within the image. We train part-attribute linear classifiers on these spatially-arranged extracted features, where the part locations are used from the training annotations. Once the network has been trained, we run max-pooling over the $14\times 14$ grid to extract features that are used in the second stage training. Note that this setup is different from the original Concept Bottleneck~\citep{koh2020concept} architecture, as the original setup predicted the keypoints from the entire image, instead of a spatially-supervised version. We use this architecture as the original setup overfit very easily.

\vspace{1.5em}
\subsubsection{Common Objects Dataset Details}
\label{sec:appendix_w3ig_details}
We construct a dataset by collecting public images from Instagram\footnote{\href{https://www.instagram.com}{\ttfamily www.instagram.com}}, involving common household objects (e.g., bed, stove, saucer, tables, etc.) with their bounding box annotations. For each bounding box, we collect 2,618 interpretable annotations, namely, parts (e.g., leg, handle, top, etc.), attributes (e.g., colors, textures, shapes, etc.), and part-attribute compositions. The dataset has 2,645,488 training and 58,525 validation samples across 115 classes with 2,618 interpretable concepts.

\vspace{1.5em}
\subsubsection{Hyperparameters}
\label{sec:appendix_hyperparameters}

Gradient-descent based methods are trained with the Adam with decoupled weight decay (AdamW) optimizer~\cite{loshchilov2017decoupled}. All methods are trained for 1000 epochs on the California Housing and FICO HELOC datasets, 100 epochs on iNaturalist Birds and Common Objects datasets, and 500 epochs on the CoverType, Newsgroups and CUB datasets. We use a cosine annealing learning rate across all gradient-descent based approaches and do a random hyperparameter search within the following specified grids.

\paragraph{Linear (Order 1 and Order 2).}
The learning rate range is $[1\mathrm{e}{-}5, 100]$, and the weight decay interval is $[1\mathrm{e}{-}13, 1.0)$. We use the same interval for $L_1-$regularized models as well.

\paragraph{Deep Neural Networks.}
We experiment with 3 neural network architectures that are all comprised of successive fully-connected layers. The first has 5 hidden layers of the shapes $[128, 128, 64, 64, 64]$ units; the second is three hidden layers deep with $[1024, 512, 512]$ units, and finally we have a shallow but wide architecture having 2 hidden layers with $[2048, 1024]$ units. We report the best accuracies achieved across all three architectures. Moreover, we observed that increasing the depth with more layers did not improve accuracy. 
We search the initial learning rate in the interval $[10^{-6}, 10.0)$, weight decay in  $[1\mathrm{e}{-}9, 1.0)$, dropout in the set $\{0, 0.05, 0.1, 0.2, 0.3, 0.4, 0.5, 0.6, 0.7, 0.8, 0.9\}$.

\paragraph{Neural Additive Models.}
We proceed with the standard proposed architectures in the original implementation of~\citet{agarwal2021neural}. We experiment with two architectures: the first is an MLP with 3 hidden layers and $[64, 64, 32]$ units, and the second is an MLP with 1 hidden layer with $1,024$ units and an ExU activation. We tune the learning rate in the interval $[1\mathrm{e}{-}5, 1.0)$, weight decay in $[1\mathrm{e}{-}9, 1.0)$, output penalty weight in $[1\mathrm{e}{-}{6}, 100)$, dropout and feature dropout in the set $\{0, 0.05, 0.1, 0.2, 0.3, 0.4, 0.5, 0.6, 0.7, 0.8, 0.9\}$.

\paragraph{EBMs.}
We describe the hyperparameter followed by the range of values we search from. 
\begin{itemize}[leftmargin=*,itemsep=0pt,topsep=0pt,parsep=0pt,partopsep=0pt]
    \item Maximum bins: $\{8, 16, 32, 64, 128, 256, 512\}$
    \item Number of interactions: $\{0, 2, 4, 8, 16, 32, 64, 128, 256, 512\}$ ($0$ for EBMs, $\geq 0$ for EB$^2$Ms)
    \item learning rate: $[1\mathrm{e}{-}6, 100)$
    \item maximum rounds: $\{1000, 2000, 4000, 8000, 16000\}$
    \item minimum samples in a leaf node: $\{1, 2, 4, 8, 10, 15, 20, 25, 50\}$
    \item maximum leaves: $\{1, 2, 4, 8, 10, 15, 20, 25, 50\}$
    \item binning: \{``quantile'', ``uniform'', ``quantile\_humanized''\}
    \item inner/outer bags:  $\{1, 2, 4, 8, 16, 32, 64, 128\}$.
\end{itemize}

\paragraph{XGBoost.} 
We describe the hyperparameter followed by the range of values we search from. 
\begin{itemize}[leftmargin=*,itemsep=0pt,topsep=0pt,parsep=0pt,partopsep=0pt]
    \item number of estimators: $\{1, 2, 4, 8, 10, 20, 50, 100, 200, 250, 500, 1000\}$
    \item max-depth: $\{\infty, 2, 5, 10, 20, 25, 50, 100, 2000\}$
    \item $\eta$: $[0.0, 1.0)$
    \item {\tt subsample}: $[0.0, 1.0)$
    \item {\tt colsample\_bytree}: $[0.0, 1.0)$
\end{itemize}
\paragraph{CART.}
We describe the hyperparameter followed by the range of values we search from. 
\begin{itemize}[leftmargin=*,itemsep=0pt,topsep=0pt,parsep=0pt,partopsep=0pt]
    \item criterion: \{``absolute error'', ``friedman mse'', ``squared error'', ``poisson''\}
    \item splitter: \{``best'', ``random''\}
    \item minimum samples leaf: $\{4, 8, 16, 32, 64, 128, 256\}$
    \item minimum samples split: $\{4, 8, 16, 32, 64, 128, 256\}$
\end{itemize}

\paragraph{SPAM-Linear}. For SPAM-Linear, there are 4 sets of hyperparameters: initial learning rate (tuned logarithmically from the set $[10^{-7}, 10^2]$), weight decay (tuned logarithmically from the set $[10^{-13}, 10^1]$), dropout for the singular values $\lambda$, which is tuned from the range $[0, 1]$, and the vector of ranks $\br$ which consequently specifies the degree of the polynomial. For degree 2 SPAM, we search ranks from the set $\cS = \{[25], [50], [100], [200], [250], [400], [500], [750], [800], [1000], [1200], [1400], [1600]\}$ and we for order 3, we search from the set $\cS\times\cS$ (i.e., product set of original range).

\paragraph{SPAM-Neural}. For SPAM-Neural, in addition to the 4 original hyperparameters, we search from the NAM space identical to the range employed for NAM (please see above). Regarding the original hyperparameters, we search: initial learning rate (tuned logarithmically from the set $[10^{-7}, 10^2]$), weight decay (tuned logarithmically from the set $[10^{-13}, 10^1]$), dropout for the singular values $\lambda$, which is tuned from the range $[0, 1]$, and the vector of ranks $\br$ which consequently specifies the degree of the polynomial. For degree 2 SPAM, we search ranks from the set $\cS = \{[25], [50], [100], [200], [250], [400], [500], [750], [800], [1000], [1200], [1400], [1600]\}$ and we for order 3, we search from the set $\cS\times\cS$ (i.e., product set of original range).
\newpage
\section{Human Subject Evaluation Details}
\subsection{CUB-200 Dataset Details}
\label{sec:appendix_amt_details}
\begin{figure}[h]
    \centering
    \fbox{\includegraphics[width=0.8\textwidth]{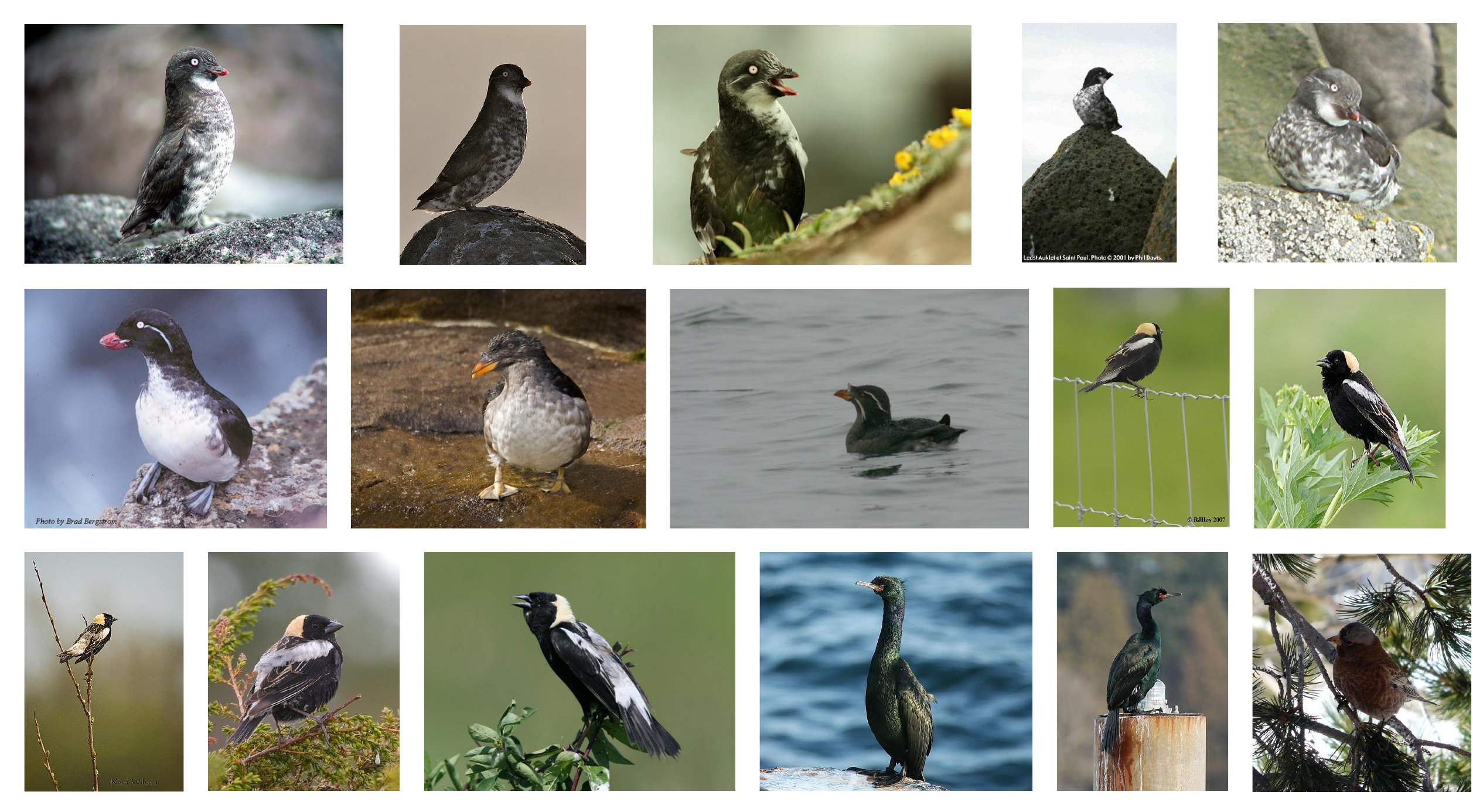}}
    \caption{Random sample Images from the CUB-200 dataset used in Human Subject Evaluations\label{fig:appendix_cub_samples}}
\end{figure}

Figure~\ref{fig:appendix_cub_samples} denotes a random sample of images that were shown to the participants. We randomly sampled images from a pre-selected set of a pair of classes. There were a total of 5 pairs, corresponding to the following classes:
\begin{enumerate}
    \item ``122.Harris\_Sparrow", ``060.Glaucous\_winged\_Gull"
\item ``057.Rose\_breasted\_Grosbeak", ``034.Gray\_crowned\_Rosy\_Finch"
\item ``066.Western\_Gull", ``035.Purple\_Finch"
\item ``064.Ring\_billed\_Gull", ``008.Rhinoceros\_Auklet"
\item ``095.Baltimore\_Oriole", ``086.Pacific\_Loon"
\end{enumerate}
All images were sampled from the validation set.

\subsection{Experiment Interface}

The experiment was implemented using the {\tt Mephisto}~\citep{miller2017parlai} library in {\tt react.js}. Please see the subsequent pages for screenshots of the exact interface.
\begin{figure}[h]
    \centering
    \fbox{\includegraphics[width=\textwidth]{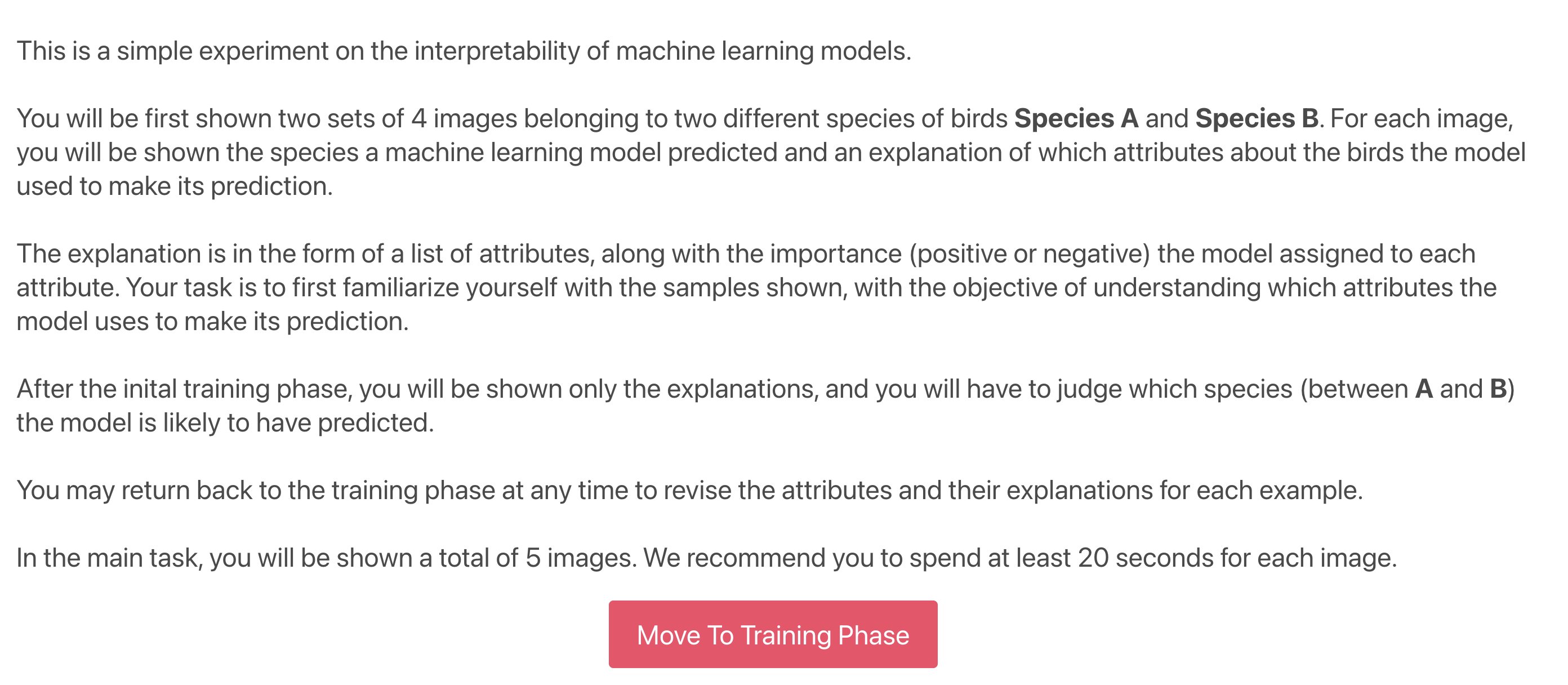}}
    \caption{Human Subject Experiment: Introduction page\label{fig:amt_intro}}
\end{figure}
\begin{figure}[h]
    \centering
    \fbox{\includegraphics[width=\textwidth]{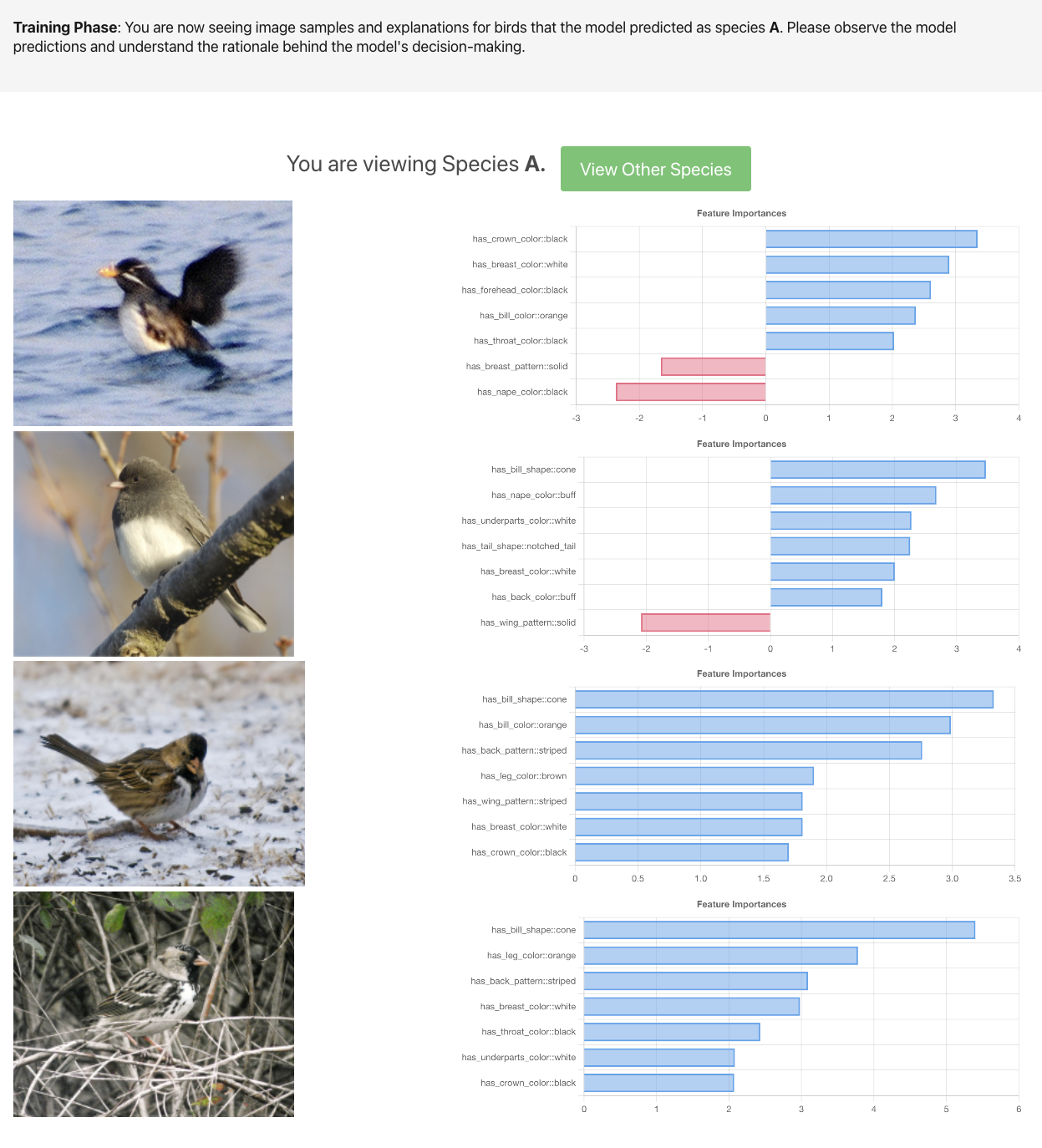}}
    \caption{Human Subject Experiment: Example training phase, class {\bf A} for linear explanations.\label{fig:amt_linear_class_a}}
\end{figure}
\begin{figure}[h]
    \centering
    \fbox{\includegraphics[width=\textwidth]{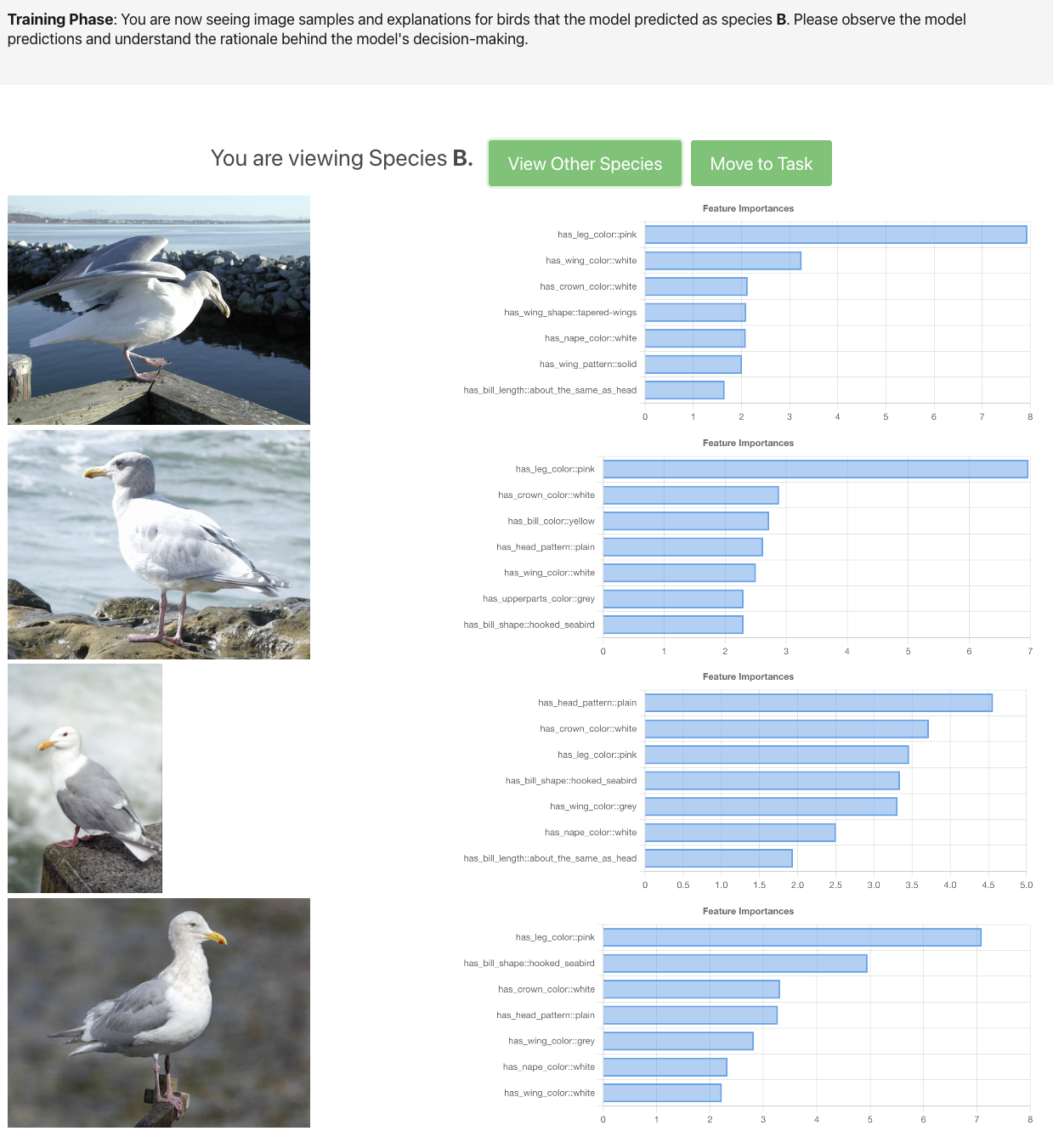}}
    \caption{Human Subject Experiment: Example training phase, class {\bf B} for linear explanations.\label{fig:amt_linear_class_b}}
\end{figure}
\begin{figure}[h]
    \centering
    \fbox{\includegraphics[width=\textwidth]{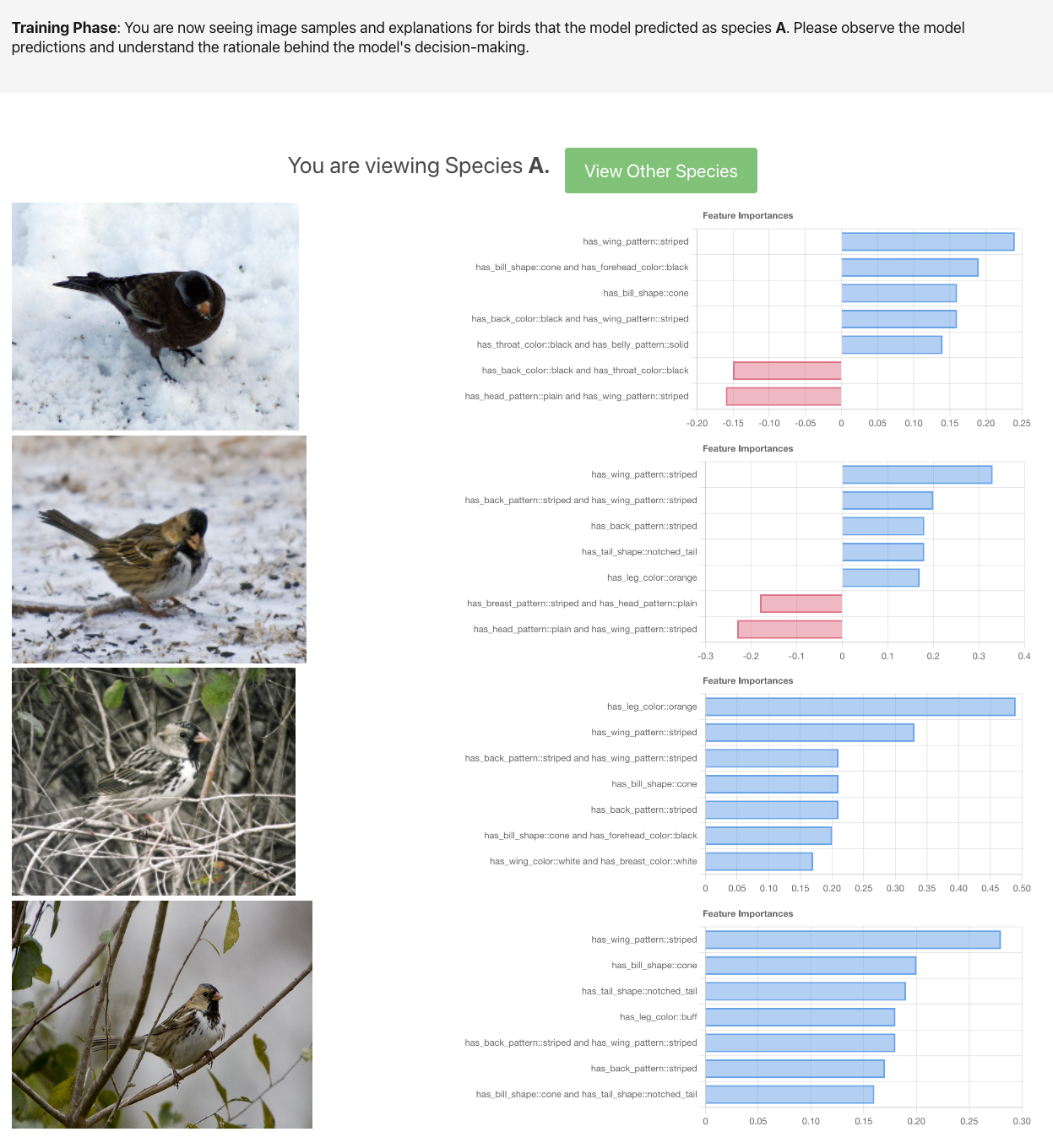}}
    \caption{Human Subject Experiment: Example training phase, class {\bf A} for SPAM (order 2) explanations.\label{fig:amt_quad_class_a}}
\end{figure}
\begin{figure}[h]
    \centering
    \fbox{\includegraphics[width=\textwidth]{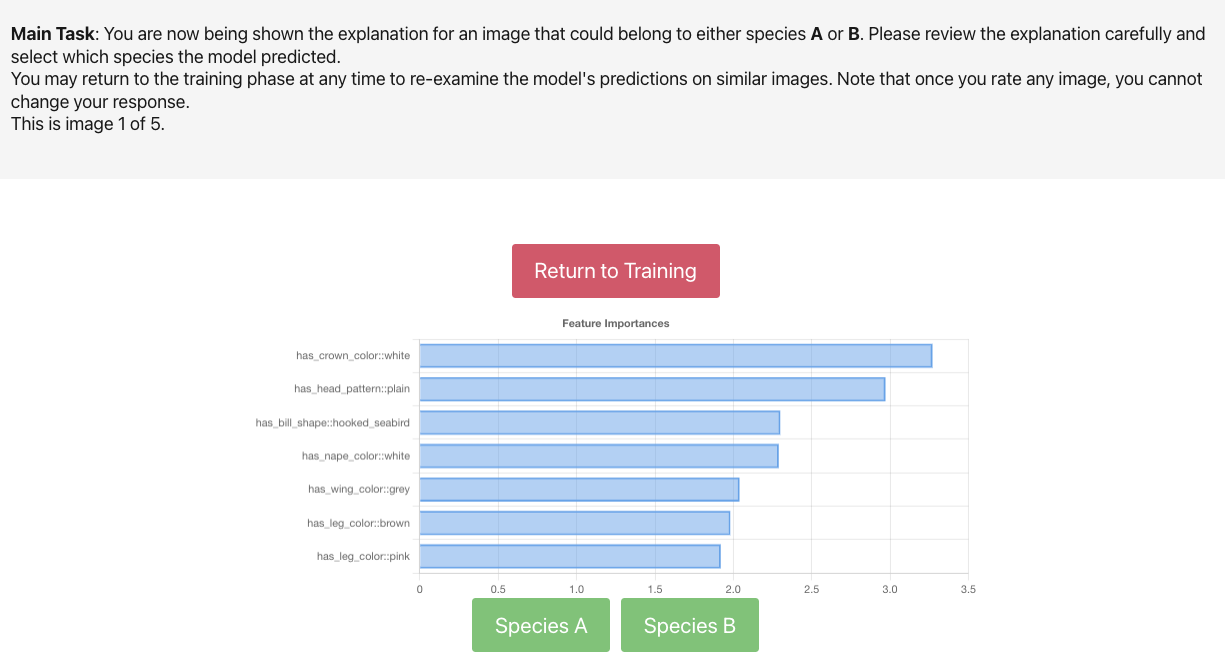}}
    \caption{Human Subject Experiment: Example testing phase, for linear explanations.\label{fig:amt_linear_test}}
\end{figure}
\begin{figure}[h]
    \centering
    \fbox{\includegraphics[width=\textwidth]{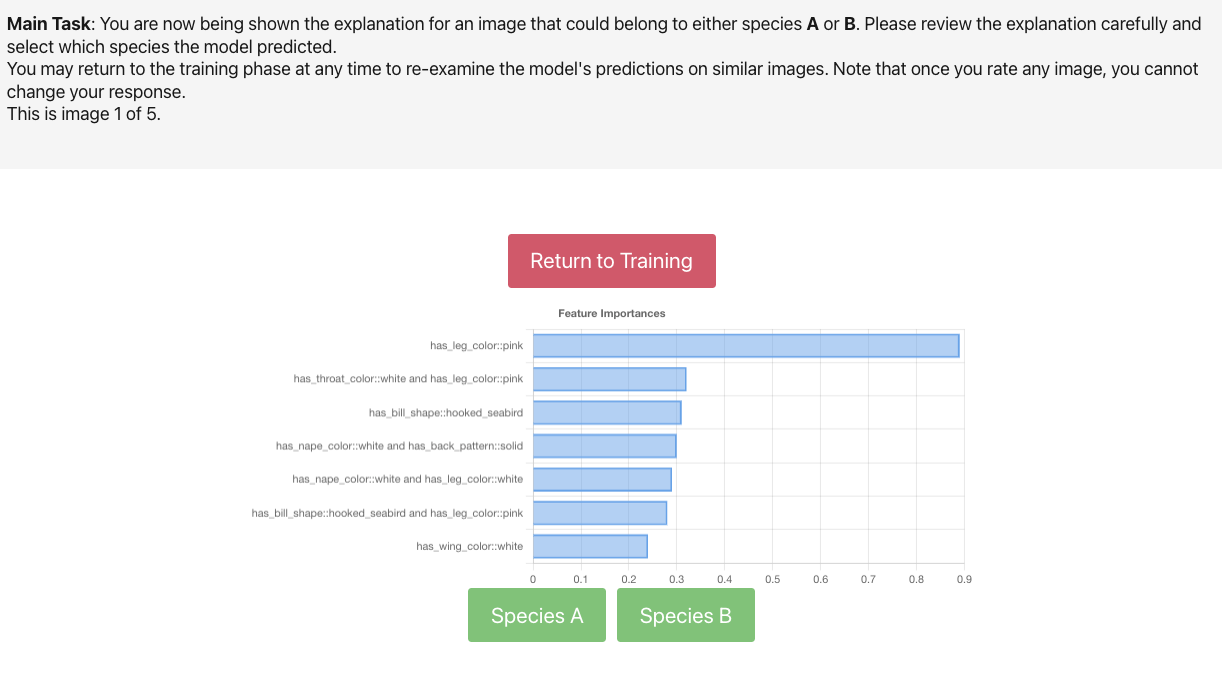}}
    \caption{Human Subject Experiment: Example testing phase, for SPAM explanations.\label{fig:amt_quad_test}}
\end{figure}

\end{document}